\documentclass[journal,onecolumn,compsoc]{IEEEtran}

\ifCLASSINFOpdf
\else
\fi

\usepackage{subfigure}
\usepackage{varwidth}
\usepackage[cmex10]{amsmath}
\usepackage{amsfonts}
\usepackage{amssymb}
\usepackage{amsthm}
\usepackage{graphicx}
\usepackage{verbatim}
\usepackage{color}
\usepackage[table]{xcolor}

\usepackage{url}
\usepackage{algorithm}
\usepackage[noend]{algpseudocode}

\newcommand{\lsize}{n}
\newcommand{\osize}{m}
\newcommand{\qsize}{q}
\newcommand{\usize}{u}
\newcommand{\vtsize}{v}
\newcommand{\testsize}{t}

\newcommand{\odsize}{d}
\newcommand{\qdsize}{r}

\newcommand{\kernelfk}{k}
\newcommand{\kernelfg}{g}

\newcommand{\lossfunction}{\mathcal{L}}

\newcommand{\regparam}{\lambda}

\newcommand{\idmatrix}{\bm{I}}

\newcommand{\transpose}{^\textnormal{T}}

\newcommand{\bm}[1]{\mathbf{#1}}

\newcommand{\objset}{D}
\newcommand{\taskset}{T}

\newcommand{\krf}{\bm{X}}

\newcommand{\eqn}[1]{\begin{align}#1\end{align}}

\newcommand{\mrows}{a}
\newcommand{\mcols}{b}
\newcommand{\nrows}{c}
\newcommand{\ncols}{d}
\newcommand{\ulen}{f}
\newcommand{\vlen}{e}

\newtheorem{lemma}{Lemma}

\newtheorem{theorem}{Theorem}
\newtheorem{definition}{Definition}

\newtheorem{remark}{Remark}

\newcommand{\tsize}{n}

\newcommand{\coeffs}{\bm{w}}
\newcommand{\dcoeffs}{\bm{a}}

\newcommand{\svecs}{\mathcal{S}}

\newcommand{\databvecs}{\mathcal{D}}
\newcommand{\taskbvecs}{\mathcal{E}}
\newcommand{\inputbvecs}{\mathcal{F}}

\newcommand{\funspace}{\mathcal{H}}
\DeclareMathOperator*{\argmin}{argmin}

\begin{document}

\title{Fast Kronecker product kernel methods via generalized vec trick}

\author{Antti Airola, Tapio Pahikkala
\thanks{A. Airola and T. Pahikkala are with the Department of Future Technologies,
University of Turku, Finland (e-mail: firstname.lastname@utu.fi).}}

%



\maketitle

\begin{abstract}
Kronecker product kernel provides the standard approach in the kernel methods literature for learning from graph data, where edges are labeled and both start and end vertices have their own feature representations. The methods allow generalization to such new edges, whose start and end vertices do not appear in the training data, a setting known as zero-shot or zero-data learning. Such a setting occurs in
numerous applications, including drug-target interaction prediction, collaborative filtering and information retrieval. Efficient training algorithms based on the so-called vec trick,
that makes use of the special structure of the Kronecker product, are known for the case where the training data is a complete bipartite graph. In this work we generalize these results to non-complete training graphs.  This allows us to derive a
general framework for training Kronecker product kernel methods, as specific examples we implement Kronecker ridge regression and support vector machine algorithms. Experimental results demonstrate that the proposed approach leads to accurate models, while allowing order of magnitude improvements in training and prediction time.

\end{abstract}



\section{Introduction}\label{introsection}

This work concerns the problem of learning supervised machine learning models from labeled bipartite graphs. 
Given a training set $(\bm{d}_i,\bm{t}_j,y_h)_{h=1}^\tsize$ of edges $(\bm{d}_i,\bm{t}_j)$, where $\bm{d}_i$ is the start vertex, and $\bm{t}_j$ the end vertex and $y_h$ the label, the goal is to learn to predict labels for new edges with unknown labels. We assume that both the start and end vertices have their own feature representations.  Further, we assume that the same vertices tend to appear as start or end vertices in multiple edges (for example, $(\bm{d}_t,\bm{t}_v)$, $(\bm{d}_u,\bm{t}_w)$, $(\bm{d}_t,\bm{t}_w)$ and $(\bm{d}_u,\bm{t}_v)$ might all belong to the same training set). This latter property is what differentiates this learning setting from the standard supervised learning setting, as the data violates the very basic i.i.d. assumption. This overlapping nature of the data has major implications both on the correct way to estimate the generalization performance of learned predictors, and on how to design efficient algorithms for learning from such data. 

As discussed by \cite{park2012flaws,Schrynemackers2013,pahikkala2015dti}, we may divide this type of learning problem into a number of distinct settings, depending whether the training and test graphs are vertex disjoint or not.
In this work, we assume that the training and test graphs are vertex disjoint, that is, neither the start vertex $\bm{d}$ nor the end vertex $\bm{t}$ of a test edge is part of any edge in the training set. This is in contrast to, for example, the less challenging setting usually considered in recommender systems literature, in which the training and test graphs are only edge disjoint. There the task is to impute missing labels to customers (start vertex) $\times$ products (end vertex) matrix, where each row and column already contains at least some known labels (see Section~\ref{sec:relwork} and \cite{menon2010loglinear} for more discussion). Following works such as \cite{palatucci2009zero,bernardino2015emabarrassing}, we use the term zero-shot learning to describe our setting, as it requires generalization to new graphs that are vertex disjoint with the training graph. This setting is illustrated in Figure~\ref{fig:setting}, where zero-shot learning corresponds to generalizing from a partially observed start vertices $\times$ end vertices label matrix to predicting unknown labels at the bottom right for edges connecting new start and end vertices that have no edges connected to them in the training set.

\begin{figure}
\includegraphics[width=0.99\linewidth]{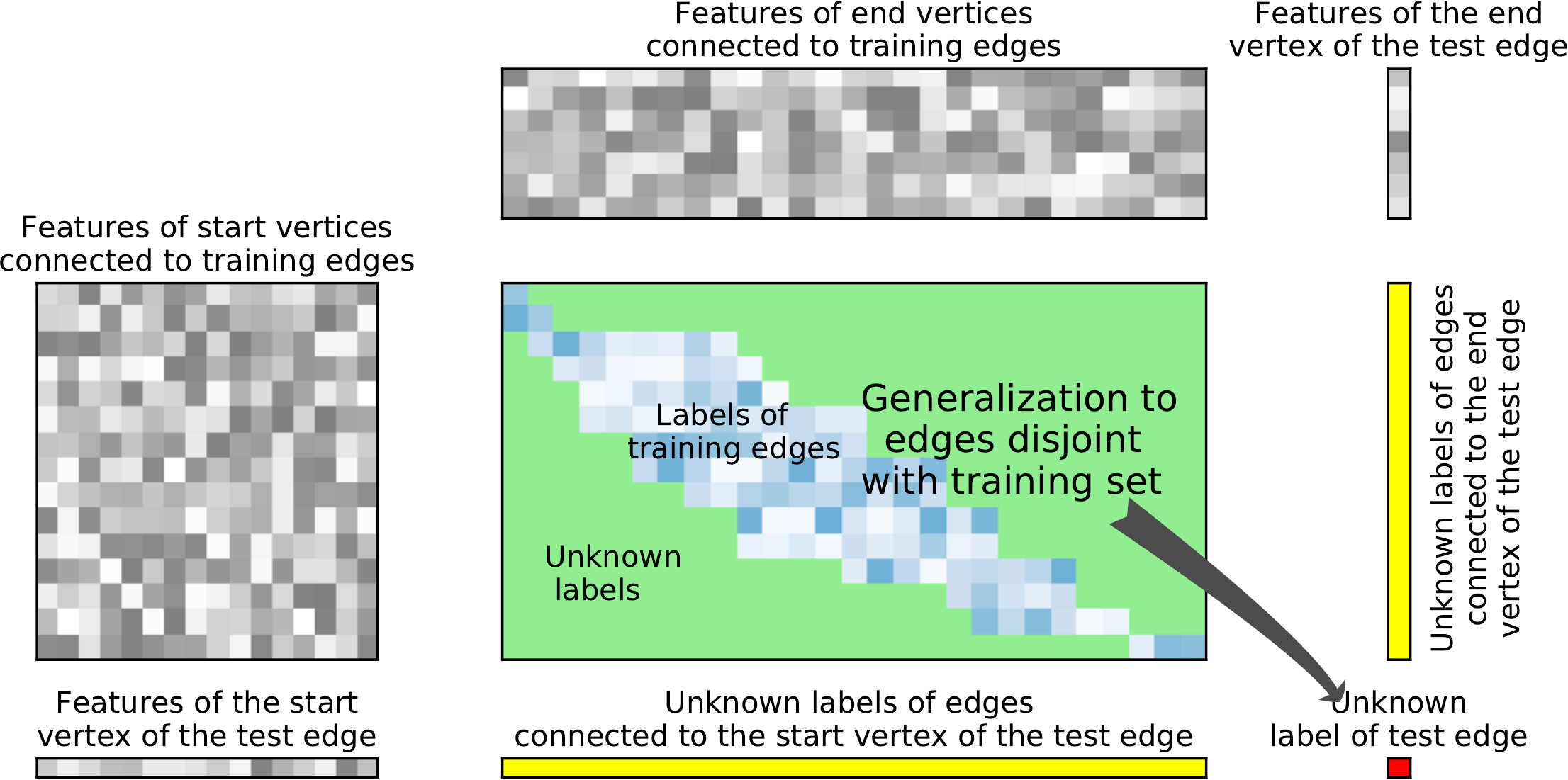}
\caption{Zero-shot learning. A predictor trained on edges and associated correct labels generalizes to such a test edge, for which both the start and end vertices are not part of the training set.}
\label{fig:setting}
\end{figure}

We consider the setting, where both the feature representations of the start and end vertices are defined implicitly via kernel functions \cite{muller2001introduction}. Further, following works such as \cite{basilico2004unifying,Benhur2005,Kashima2009linkprob, park2009pairwise,hue2010on,Raymond2010scalable,pahikkala2010conditional,waegeman2012learninggraded,pahikkala2013efficient,pahikkala2014twostep,bernardino2015emabarrassing}, we assume that the feature representation of an edge is defined as the product of the start and end vertex kernels, resulting in the so-called Kronecker product kernel. Kronecker product kernel methods can be trained by plugging this edge kernel to any existing kernel machine solver. However, for problems where the number of edges is large this is not feasible in practice, as typically the runtime and in many cases also the memory use of kernel solvers grows at least quadratically with respect to number of edges. 

In this work, we propose the first general Kronecker product kernel based learning algorithm for large-scale problems. This is realized by generalizing the computational short-cut implied by Roth's column lemma \cite{roth1934columnlemma} (often known as the vec-trick). As an example, we demonstrate in the experiments how the approach outperforms existing state-of-the-art SVM solver by several orders of magnitude when using the Gaussian kernel, the most well-known special case of the Kronecker product kernel. Further, we show that the proposed methods perform well in comparison to other types of graph prediction methods. This work extends our previous work \cite{pahikkala2014fastgradient}, that specifically considered the case of efficient ridge regression with Kronecker product kernel.



\subsection{Related work}\label{sec:relwork}

The bipartite graph learning problem considered in this work appears under various names in the machine learning literature, including link prediction \cite{Kashima2009linkprob}, relational learning \cite{waegeman2012learninggraded}, pair-input prediction  \cite{park2012flaws}, and dyadic prediction \cite{menon2010loglinear}. The setting provides a unifying framework for prominent machine learning application areas in biology and chemistry dealing with pairwise interaction predictions, such as drug-target \cite{pahikkala2015dti}, and protein-RNA interaction prediction. The setting is prominent also in collaborative filtering, where the central task is to predict ratings or rankings for (customer, product) -pairs (see e.g. \cite{basilico2004unifying,koren2009matrix,menon2010loglinear}). Also, in information retrieval the reranking of search engine results has been popularly cast as a bipartite graph ranking problem over (query, document) -pairs \cite{Liu2011letorir}. Many social network prediction tasks such as matching people to advertisements fall also naturally within the considered framework. Applications where the entities whose relations are predicted are of same type (i.e. protein-protein interactions, friendship in social network) can also be modeled as bipartite graph learning problems. In this case two vertices are used to represent each entity, one included in the set of start vertices and one in the set of end vertices.

In this work we consider the supervised graph learning problem, where the training edges are labeled. A limited range of graph learning problems may also be solved using unsupervised clustering methods \cite{peng2016unified,peng2016constructing}, for example by predicting that a relation holds between nodes assigned to the same cluster.

Kernel methods \cite{muller2001introduction} are among the most widely used approaches in machine learning, both due to their strong foundations in statistical learning theory, as well as their practical applicability for a wide range of learning problems. The methods represent the data using positive semidefinite kernel functions, that allow incorporating prior knowledge about the problem domain, provide implicit non-linear mappings, and may be learned directly from data (see e.g. \cite{Jayadeva2009kernel}). Recently, multi-task learning has been a topic of substantial research in the area of kernel methods (see e.g. \cite{bonilla2007taskspecific,jayadeva2008derivatives, alvarez2012kernels}).

The use of Kronecker product kernel for graph learning problems was originally proposed at the same time in several independent works for collaborative filtering \cite{basilico2004unifying} protein-protein interaction prediction \cite{Benhur2005} and entity resolution \cite{oyama2004}. Since then
the approach has become especially popular in predicting biological interactions (see e.g. \cite{vanLaarhoven2011,gonen2012predicting,pahikkala2015dti}), as well as a standard building block in more theoretical work concerning the
development of multi-task learning methods (see e.g. \cite{bonilla2007taskspecific,Hayashi2012self,alvarez2012kernels}).
Concerning specifically the zero-shot learning setting, Kronecker product kernel methods have been shown to outperform a variety of baseline
methods in areas such as recommender systems \cite{park2009pairwise}, drug-target prediction \cite{pahikkala2015dti} and
image categorization \cite{bernardino2015emabarrassing}. 

Theoretical  properties of the Kronecker product kernel were recently analyzed by \cite{waegeman2012learninggraded}. They showed that if the start and end vertex kernels have the so-called universality property (e.g. Gaussian kernel) \cite{steinwart2005consistency}, then the resulting Kronecker edge kernel is also universal, resulting in universal consistency when training kernel-based learning algorithms, such as support vector machines and ridge regression using the kernel.

Recently, it has been shown that the property that the same vertices tend to appear many times in the training data can be exploited in order to develop much faster Kronecker product kernel method training algorithms. 
Several efficient machine learning algorithms have been proposed for the special case, where the training sets consists of a complete bipartite graph, meaning that each possible start-end vertex pair appears exactly once, and a ridge regression loss is minimized. Specifically \cite{Raymond2010scalable,pahikkala2010conditional,pahikkala2013efficient,pahikkala2014twostep,bernardino2015emabarrassing} derive closed form solutions based on Kronecker algebraic optimization (see also \cite{martin2006shiftedkron} for the basic mathematical results underlying these studies). Further, iterative methods based on Kronecker product kernel matrix - vector multiplications, have been proposed (see e.g. \cite{Kashima2009linkprob,pahikkala2010conditional,pahikkala2013efficient}). 

However, the assumption about the training set graph being complete can be considered a major limitation on the applicability of these methods. As far as we are aware, similar efficient training algorithms as ours for regularized kernel methods have not been previously proposed for the case, where the training graph is sparse in the sense that only a subset of the possible edges appear in the training set. 

Much of the work done in multi-task learning can be considered to fall under the bipartite graph prediction, and even zero-shot learning framework \cite{bonilla2007taskspecific,larochelle2008zero,pahikkala2010conditional,Hayashi2012self,pahikkala2014twostep,bernardino2015emabarrassing,Schafer2015}. For example, in multi-label data one might have as training data images and labels describing the properties of the images (e.g. food, person, outside). In a traditional multi-label learning setting only the images have features describing them, and we assume that the set of possible labels is fixed in advance, and that examples of each label must appear in the training set. However, if the labels have also features describing then, we may relax these assumptions by casting the task as a graph prediction problem over (data point, label) -edges (e.g. predict 1 if label $\bm{t}_i$ should be assigned to data point $\bm{d}_j$ and -1 if not, or rank a set of candidate labels from best matching to worst). This allows predictions even for new data points and labels both not part of the training set.

There exist two related yet distinct graph learning problems, that have been studied in-depth in the machine learning literature. Let $\objset$ and $\taskset$ denote, respectively, the sets of start and end vertices connected to the training edges, and $(\bm{d},\bm{t})$ a new edge, for which the correct label needs to be predicted. First, if $\bm{d}\in\objset$ and $\bm{t}\in\taskset$, the problem becomes that of matrix completion (predicting the the unknown labels within the center rectangle in Figure~\ref{fig:setting}). Recently this setting has been considered especially in the recommender systems literature, where matrix factorization methods have become the de-facto standard approach for solving these problems (see e.g. \cite{koren2009matrix}). The latent feature representations learned by the factorization methods, however, only generalize to edges that have both vertices connected to at least one edge in the training set. The second setting corresponds to the situation where $\bm{d}\in\objset$ and $\bm{t}\notin\taskset$, or alternatively $\bm{d}\notin\objset$ and $\bm{t}\in\taskset$ (predicting the block on the right or below in Figure~\ref{fig:setting}). Here predictions are needed only for new rows or columns of the label matrix, but not both simultaneously. The problem may be solved by training a standard supervised learning method for each row or column separately, or by using multi-target learning methods that also aim to utilize correlations between similar rows or columns \cite{evgeniou2005learning}.

While simultaneous use of both start and end vertex features can and has been integrated to learning also in these two related settings, they are not essential for generalization, and can even be harmful (so called negative transfer, see e.g. \cite{pahikkala2014twostep,pahikkala2015dti}). The learning algorithms considered in this work are applicable in these related settings, provided that both start and end vertex features are available. However, simpler methods that do not need both of these information sources often provide a competitive alternative. Thus we do not consider further these two settings in this work, as these have been already quite thoroughly explored in previous literature.


\section{Generalized vec-trick}

In this section, we introduce the novel generalized Vec-trick algorithm (Algorithm~\ref{alg:bgkbvdetailed}) for computing such Kronecker products, where a submatrix of a Kronecker product matrix $(\bm{M}\otimes\bm{N})$ is multiplied to a vector.
The algorithm forms the basic building block for developing computationally efficient training methods for Kronecker product kernel methods.

First, we introduce some notation. By $[n]$, where $n\in\mathbb{N}$, we denote the index set $\{1,...,n\}$. By $\bm{A}\in\mathbb{R}^{a\times b}$ we denote an $a\times b$ matrix, and by $A_{i,j}$ the $i,j$:th element of this matrix. By $\textnormal{vec}(\bm{A})$ we denote the vectorization of $\bm{A}$, which is the $ab\times 1$ column vector obtained by stacking all the columns of $\bm{A}$ in order starting from the first column. $\bm{A}\otimes\bm{C}$ denotes the Kronecker product of $\bm{A}$ and $\bm{C}$. By $\bm{a}\in\mathbb{R}^{c}$ we denote a column vector of size $c\times 1$, and by $a_i$ its $i$:th element.

There exist several studies in the machine learning literature in which the systems of linear equations involving Kronecker products have been accelerated with the so-called ``vec-trick''. This is characterized by the following result known as the Roth's column lemma in the Kronecker product algebra:
\begin{lemma}[Roth's column lemma; "Vec trick" \cite{roth1934columnlemma}]\label{veclemma}
Let $\bm{M}\in\mathbb{R}^{a\times b}$, $\bm{Q}\in\mathbb{R}^{b\times c}$, and $\bm{N}\in\mathbb{R}^{c\times d}$ be matrices. Then,
\eqn{\label{veclemmaeq}
(\bm{N}\transpose\otimes\bm{M})\textnormal{vec}(\bm{Q})=\textnormal{vec}(\bm{M}\bm{Q}\bm{N}).
}
\end{lemma}
\noindent It is obvious that the right hand side of (\ref{veclemmaeq}) is considerably faster to compute than the left hand side, because it avoids the direct computation of the large Kronecker product.

\begin{algorithm}[t]
  \begin{algorithmic}[1]
    \Require $\bm{M}\in\mathbb{R}^{\mrows\times\mcols},\bm{N}\in\mathbb{R}^{\nrows\times\ncols},\bm{v}\in\mathbb{R}^{\vlen},\bm{p}\in[\mrows]^\ulen,\bm{q}\in[\nrows]^\ulen,\bm{r}\in[\mcols]^\vlen,\bm{t}\in[\ncols]^\vlen$
    \If{$\mrows\vlen+\ncols\ulen<\nrows\vlen+\mcols\ulen$}
      \State $\bm{T}\gets\bm{0}\in\mathbb{R}^{\ncols\times\mrows}$
      \For {$h=1,\ldots,\vlen$}
        \State $i,j\gets r_h,t_h$
        \For {$k=1,\ldots,\mrows$}
          \State $T_{j,k}\gets T_{j,k}+v_{h}M_{k, i}$
        \EndFor
      \EndFor
      \State $\bm{u}\gets\bm{0}\in\mathbb{R}^{\ulen}$
      \For {$h=1,\ldots,\ulen$}
        \State $i,j\gets p_h,q_h$
        \For {$k=1,\ldots,\ncols$}
          \State $u_h\gets u_h+N_{j,k}T_{k,i}$
        \EndFor
      \EndFor
    \Else
      \State $\bm{S}\gets\bm{0}\in\mathbb{R}^{\nrows\times\mcols}$
      \For {$h=1,\ldots,\vlen$}
        \State $i,j\gets r_h,t_h$
        \For {$k=1,\ldots,\nrows$}
          \State $S_{k,i}\gets S_{k,i}+v_{h}N_{k,j}$
        \EndFor
      \EndFor
      \State $\bm{u}\gets\bm{0}\in\mathbb{R}^{\ulen}$
      \For {$h=1,\ldots,\ulen$}
        \State $i,j\gets p_h,q_h$
        \For {$k=1,\ldots,\mcols$}
          \State $u_h\gets u_h+S_{j,k}M_{i,k}$
        \EndFor
      \EndFor
    \EndIf
    \State \textbf{return} $\bm{u}$
  \caption{Generalized Vec trick: $\bm{u}\gets\bm{R}(\bm{M}\otimes\bm{N})\bm{C}\transpose\bm{v}$}
  \label{alg:bgkbvdetailed}
  \end{algorithmic}
\end{algorithm}

We next shift our consideration to matrices that correspond to a submatrix of a Kronecker product matrix. First, we introduce some further notation. To index certain of rows or columns of a matrix, the following construction may be used:
\begin{definition}[Index matrix]
Let $\bm{M}\in\mathbb{R}^{\mrows\times\mcols}$ and let $\bm{s}=(s_1,\ldots,s_\ulen)\transpose\in[\mrows]^\ulen$ be a sequence of $\ulen$ row indices of $\bm{M}$. We say that
\[
\bm{S}=\left(
\begin{array}{c}
\bm{e}_{s_1}\transpose\\
\vdots\\
\bm{e}_{s_\ulen}\transpose
\end{array}\right)\;,
\]
where $\bm{e}_i$ is the $i$th standard basis vector of $\mathbb{R}^a$, is a row indexing matrix for $\bm{M}$ determined by $\bm{s}$. The column indexing matrices are defined analogously.
\end{definition}

The above type of ordinary indexing matrices can of course be used with Kronecker product matrices as well. However, the forthcoming computational complexity considerations require a more detailed construction in which we refer to the indices of the factor matrices rather than to those of the product matrix. This is characterized by the following lemma which follows directly from the properties of the Kronecker product:
\begin{lemma}[Kronecker product index matrix]
\label{def:kronsampler}
Let $\bm{M}\in\mathbb{R}^{\mrows\times\mcols}$, $\bm{N}\in\mathbb{R}^{\nrows\times\ncols}$ and let $\bm{S}$ be an index matrix for the Kronecker product $\bm{M}\otimes\bm{N}$. Then, $\bm{S}$ can be expressed as
\[
\bm{S}=\left(
\begin{array}{c}
\bm{e}_{(p_1-1)\nrows+q_1}\transpose\\
\vdots\\
\bm{e}_{(p_\ulen-1)\nrows+q_\ulen}\transpose
\end{array}\right)\;,
\]
where $\bm{p}=(p_1,\ldots,p_\ulen)\transpose\in[\mrows]^\ulen$ and $\bm{q}=(q_1,\ldots,q_\ulen)\transpose\in[\nrows]^\ulen$ are sequences of row indices of $\bm{M}$ and $\bm{N}$, respectively. The entries of $\bm{p}$ and $\bm{q}$ are given as 
\begin{align}
p&=\lfloor (i-1)/\nrows\rfloor+1\phantom{WW}q=(i-1)\mod \nrows+1\;,\label{quotientremainder}
\end{align}
where $i\in[ac]$ denotes a row index of $\bm{M}\otimes\bm{N}$, and $p\in[a]$ and $q\in[c]$ map to the corresponding rows in $\bm{M}$ and $\bm{N}$, respectively. The column indexing matrices are defined analogously.
\end{lemma}
\begin{proof}
The claim is an immediate consequence of the definition of the Kronecker product, where the relationship between the row indices of $\bm{M}$ and $\bm{N}$ with those of their Kronecker product is exactly the one given in (\ref{quotientremainder}).
\end{proof}
\begin{theorem}\label{mainproposition}
Let $\bm{M}\in\mathbb{R}^{\mrows\times\mcols},\bm{N}\in\mathbb{R}^{\nrows\times\ncols}$, and let $\bm{R}\in\{0,1\}^{\ulen\times\mrows\nrows}$ and $\bm{C}\in\{0,1\}^{\vlen\times\mcols\ncols}$ to be, respectively, a row index matrix and a column index matrix of $\bm{M}\otimes\bm{N}$, such that $\bm{R}$ is determined by the sequences $\bm{p}=(p_1,\ldots,p_\ulen)\transpose\in[\mrows]^\ulen$ and $\bm{q}=(q_1,\ldots,q_\ulen)\transpose\in[\nrows]^\ulen$, and $\bm{C}$ by  $\bm{r}=(r_1,\ldots,r_\vlen)\transpose\in[\mcols]^\vlen$ and $\bm{t}=(t_1,\ldots,t_\vlen)\transpose\in[\ncols]^\vlen$. Further, we assume that the sequences $\bm{p}$, $\bm{q}$, $\bm{r}$ and $\bm{t}$ when considered as mappings, are all surjective. Further, let $\bm{v}\in\mathbb{R}^{\vlen}$.
The product
\eqn{\label{vecpropeq}
\bm{R}(\bm{M}\otimes\bm{N})\bm{C}\transpose\bm{v}
}
can be computed in $O(\min(\mrows\vlen+\ncols\ulen,\nrows\vlen+\mcols\ulen))$ time using the generalized Vec trick (Algorithm~\ref{alg:bgkbvdetailed}).
\end{theorem}
\begin{proof}
Since the sequences $\bm{p}$, $\bm{q}$, $\bm{r}$ and $\bm{t}$,  determining $\bm{R}$ and $\bm{C}$ are all surjective on their co-domains, we imply that $\max(\mrows,\nrows)\leq\ulen$ and $\max(\mcols,\ncols)\leq\vlen$.
Let $\bm{V}\in{R}^{\ncols\times\mcols}$ be a matrix such that $\textnormal{vec}(\bm{V})=\bm{C}\transpose\bm{v}$. Then, according to Lemma~\ref{veclemma}, (\ref{vecpropeq}) can be rewritten as
\[
\bm{R}\textnormal{vec}(\bm{N}\bm{V}\bm{M}\transpose)\;,
\]
which does not involve a Kronecker product. Calculating the right hand side can be started by first computing either $\bm{S}=\bm{N}\bm{V}$ or $\bm{T}=\bm{V}\bm{M}\transpose$ requiring $O(\nrows\vlen)$ and $O(\mrows\vlen)$ time, respectively, due to $\bm{V}$ having only at most $\vlen$ nonzero entries. Then, each of the $\ulen$ elements of the product can be computed either by calculating the inner product between a row of $\bm{S}$ and a row of $\bm{M}$ or between a row of $\bm{N}$ and a column of $\bm{T}$, depending whether $\bm{S}$ or $\bm{T}$ was computed. The former inner products require $\mcols$ and the latter $\ncols$ flops, and hence the overall complexity is $O(\min(\mrows\vlen+\ncols\ulen,\nrows\vlen+\mcols\ulen))$. The pseudocode for the algorithm following the ideas of the proof is presented in Algorithm~\ref{alg:bgkbvdetailed}. We note that the index matrices $\bm{R}$ and $\bm{C}$ are not explicitly given as input to the algorithm, because the elements of these sparse matrices are encoded in the vectors $\bm{p}$, $\bm{q}$, $\bm{r}$ and $\bm{t}$.
\end{proof}
\begin{remark}[Vec trick as special case]
If $\bm{R}=\bm{C}=\bm{I}$, the generalized vec trick reduces to the standard vec trick, and the complexity of Algorithm~\ref{alg:bgkbvdetailed} is $O(\min(\mrows\mcols\ncols+\mrows\nrows\ncols, \mcols\nrows\ncols+\mrows\mcols\nrows))$, the same as for the standard vec trick in Lemma~\ref{veclemmaeq}.
\end{remark}


\section{General framework for Kronecker product kernel methods}

In this section, we connect the above abstract considerations with practical learning problems with bipartite graph data. We start by defining some notation which will be used through the forthcoming considerations. As observed from Figure~\ref{fig:setting}, the relevant data dimensions are the sizes of the two data matrices and the number of edges with known labels. The notations are summarized in Table~\ref{tb:notation}.

\begin{table}[t]
\caption{Notation concerning training set dimensions.}
\centering
\begin{center}
{\setlength{\tabcolsep}{3pt}
\begin{tabular}{|ll|}
\hline
$\lsize$& number of labeled edges in training set\\
$\osize$& number of unique start vertices connected to training edges\\
$\qsize$& number of unique end vertices connected to training edges\\
$\odsize$& number of start vertex features\\
$\qdsize$& number of end vertex features\\
\hline
\end{tabular}
}
\end{center}
\label{tb:notation}
\end{table}


The training set consists of a labeled sequence of edges $(\bm{d}_{r_h},\bm{t}_{s_h},y_h)_{h=1}^\lsize\in(\mathcal{D}\times\mathcal{T}\times\mathcal{Y})^\lsize$ of a labeled bipartite graph, each edge consisting of a start vertex $\bm{d}_{r_h}\in\mathcal{D}$, end vertex $\bm{t}_{s_h}\in\mathcal{T}$ and an associated label $y_h$, where $\mathcal{D}$ and $\mathcal{T}$ are the spaces of start and end vertices, respectively, and $\mathcal{Y}$ is the label space ($\mathcal{Y}=\mathbb{R}$ for regression and $\mathcal{Y}=\{-1,1\}$ for binary classification). Moreover, let $\{\bm{d}_i\}_{i=1}^\osize\subset\mathcal{D}$ and $\{\bm{t}_i\}_{i=1}^\qsize\subset\mathcal{T}$ denote the sets of such start vertices and end vertices, respectively, that are encountered in the training set as part of at least one edge. Further, let $\bm{r}=(r_1,\ldots,r_\lsize)\transpose\in[\qsize]^\lsize$ and $\bm{s}=(s_1,\ldots,s_\lsize)\transpose\in[\osize]^\lsize$ be index sequences that map the training edges to their corresponding start and end vertices, respectively. Looking at Figure~\ref{fig:setting}, one may think of $r_i$ and $s_i$ as row and column indices for the known entries in the start vertices times end vertices matrix, that correspond to the training data.

Let us assume a randomly drawn new sequence of edges $(\bm{d}_i,\bm{t}_j)_{h=1}^\testsize\in(\mathcal{D}\times\mathcal{T})^\testsize$, where $1<i<\usize$, $1<j<\vtsize$, and the labels are unknown. Our aim is to learn from the training set a prediction function $f:\mathcal{D}\times\mathcal{T}\rightarrow\mathcal{Y}$, such that can correctly predict the labels of such edges. We further assume that the graphs corresponding to the train and test data are completely disconnected, that is, the sets of their start vertices are mutually disjoint and so are the sets of their end vertices. This is the main difference of our setting to the types of problems solved typically for example within the recommender systems literature, rather our setting corresponds to the so-called zero-shot, or cold-start problem where test vertices do not appear in the training set.

Let $\kernelfk : \mathcal{D} \times \mathcal{D} \to [0,\infty)$ and $\kernelfg : \mathcal{T} \times \mathcal{T} \to [0,\infty)$ denote positive semi-definite kernel functions \cite{muller2001introduction} defined for the start and end vertices respectively. Further, the Kronecker product kernel $\kernelfk^{\otimes}(\bm{d}, \bm{t}, \bm{d}', \bm{t}')= \kernelfk(\bm{d}, \bm{d}')\kernelfg(\bm{t},\bm{t}')$ for the edges is defined as the product of these two base kernels. Let $K_{i,j}=\kernelfk(\bm{d}_i, \bm{d}_j)$ and $G_{i,j}= \kernelfg(\bm{t}_i, \bm{t}_j)$ denote the start and end vertex kernel matrices for the training data, respectively. Further, let $\bm{R}\in\{0,1\}^{\lsize\times\osize\qsize}$ be the Kronecker product index matrix determined by the start and end vertex index sequences $\bm{r}$ and $\bm{s}$, as defined in Lemma~\ref{def:kronsampler}. Then $\bm{R}(\bm{G}\otimes\bm{K})\bm{R}\transpose$ is the Kronecker product kernel matrix containing the kernel evaluations corresponding to the training edges. Note that the methods to be introduced will never explicitly form this prohibitively large Kronecker product matrix, rather kernel matrix multiplications are implemented with Algorithm~\ref{alg:bgkbvdetailed},
that requires as input $\bm{G}$ and $\bm{K}$, as well as the index sequences $\bm{r}$ and $\bm{s}$ that implicitly define $\bm{R}$.

We consider the regularized risk minimization problem
\begin{equation}\label{regalgorithm}
f^* = \argmin_{f \in\funspace} J(f) 
\end{equation}
with
\begin{equation*}
J(f) = 
\lossfunction(\bm{p},\bm{y})
+\frac{\regparam}{2}\Arrowvert f \Arrowvert_{\funspace}^2,
\end{equation*}
where $\bm{p}\in\mathbb{R}^\lsize$ and $\bm{y}\in\mathbb{R}^\lsize$ are the predicted and correct outputs for the training set,
$\lossfunction$ a convex nonnegative loss function and $\regparam>0$ a regularization parameter. Choosing as $\funspace$ the reproducing kernel Hilbert space defined by $\kernelfk^{\otimes}$, the generalized representer theorem \cite{scholkopf2001representer} guarantees that the optimal solutions are of the form
\begin{equation*}
f^*(\bm{d},\bm{t})= \sum_{i=1}^{n}a_i k(\bm{d}_{r_i},\bm{d})g(\bm{t}_{s_i},\bm{t}),
\end{equation*}
where we call $\bm{a}\in\mathbb{R}^{\lsize}$ the vector of dual coefficients.

Now let us consider the special case, where the start and end vertex spaces are real vector spaces, that is, $\mathcal{D}=\mathbb{R}^{\odsize}$ and $\mathcal{T}=\mathbb{R}^{\qdsize}$, and hence both the start and end vertices have a finite dimensional feature representation. Further, if the start and end vertex kernels are linear, the Kronecker product kernel can be written open as the inner product $\kernelfk^{\otimes}(\bm{d}, \bm{t}, \bm{d}', \bm{t}')=\langle \bm{d}\otimes\bm{t}, \bm{d}'\otimes\bm{t}'\rangle$, that is the edges have an explicit feature representation with respect to this kernel as the Kronecker product of the two feature vectors. Let $\bm{D}\in\mathbb{R}^{\osize\times\odsize}$ and $\bm{T}\in\mathbb{R}^{\qsize\times\qdsize}$, respectively, contain the feature representations of the training set start and end vertices. Now the joint Kronecker feature representation for the training data can be expressed as $\krf=\bm{R}(\bm{T}\otimes\bm{D})$. In this case we can equivalently define the regularized risk minimizer as
\begin{equation*}
f^*(\bm{d},\bm{t}) = \langle \bm{d} \otimes \bm{t}, \sum_{i=1}^{n}a_i \bm{d}_{r_i}\otimes\bm{t}_{s_i} \rangle = \langle \bm{d} \otimes \bm{t}, \bm{w}\rangle, 
\end{equation*}
where we call $\bm{w}\in\mathbb{R}^{dr}$ the vector of primal coefficients.

\subsection{Efficient prediction}

Next, we consider how predictions can be efficiently computed for graph data both with dual and primal predictors. This operation is necessary both during training when training set predictions are needed to evaluate the quality of the current predictor, as well as when applying the final trained predictor on new edges. Let us assume a new sequence of edges $(\bm{d}_i,\bm{t}_j)_{h=1}^\testsize\in(\mathcal{D}\times\mathcal{T})^\testsize$, where $1<i<\usize$, $1<j<\vtsize$, and $\bm{d}_i$ and $\bm{t}_j$ may or may not overlap with the training set start and end vertices.
Let $\widehat{\bm{K}}\in\mathbb{R}^{\usize\times\osize}$ and $\widehat{\bm{G}}\in\mathbb{R}^{\vtsize\times\qsize}$ contain the kernel evaluations between the training start and end vertices, and the new start and end vertices, respectively. Further, the Kronecker product index matrix $\widehat{\bm{R}}\in\{0,1\}^{\testsize\times\usize\vtsize}$ encodes those new edges, for which the predictions are needed for.

For the dual model, the predictions can be computed as
\begin{equation*}
\widehat{\bm{R}}(\widehat{\bm{G}}\otimes\widehat{\bm{K}})\bm{R}\transpose\bm{a}
\end{equation*}
resulting in a computational complexity of $O(\min(\vtsize\lsize+\osize\testsize,\usize\lsize+\qsize\testsize))$. Further, we note that for sparse models where a large portion of the $a_i$ coefficients have value $0$ (most notably, support vector machine predictors), we can further substantially speed up prediction by removing these entries from $a$ and $\bm{R}$ and correspondingly replace the term $\lsize$ in the complexity with the number of non-zero coefficients. In this case, the prediction complexity will be 
\begin{equation}
O(\min(\vtsize \Arrowvert\bm{a}\Arrowvert_0+\osize\testsize,\usize\Arrowvert\bm{a}\Arrowvert_0+\qsize\testsize)),
\label{eq:predcomp}
\end{equation}
where $\Arrowvert\bm{a}\Arrowvert_0$ is the zero-norm measuring the number of non-zero elements in $\bm{a}$.
This is in contrast to the explicit computation by forming the full test kernel matrix, which would result in
\begin{equation}
O(\testsize\Arrowvert\bm{a}\Arrowvert_0)
\label{eq:naivepredcomp}
\end{equation}
complexity.

In the primal case the predictions can be computed as
\begin{equation*}
\widehat{\bm{R}}(\widehat{\bm{T}}\otimes\widehat{\bm{D}})\bm{w}
\end{equation*}
resulting in a computational complexity of
$O(\min(\vtsize\odsize\qdsize+\odsize\testsize,\usize\odsize\qdsize+\qdsize\testsize))$. For high-dimensional data ($\odsize\qdsize>>\lsize$) one will save substantial computation by using the dual model instead of the primal.

\subsection{Efficient learning}

\begin{table*}[t]
	\caption{Loss functions and their (sub)gradients and (generalized) Hessians. For binary classification losses marked * we assume that $y_i\in\{-1,1\}$, while for the rest $y_i\in\mathbb{R}$.}
\centering
	\begin{tabular}{|c|c|c|c|c|c|}
    \hline
	Method & $\lossfunction$ &$\bm{g}_i$ & $\bm{H}_{i,i}$
& $\bm{H}_{i,j}, i\neq j $\\
	\hline
	Ridge \cite{hoerl1970ridge} &  $\frac{1}{2}\sum_{i=1}^\lsize (p_i-y_i)^2$ & $p_i-y_i$ &  1 & 0\\
    L1-SVM* \cite{vapnik1995statistical} & $\sum_{i=1}^\lsize\max(0, 1-p_i \cdot y_i)$ & $
\begin{array}{cc}
-y_i &\textnormal{ if }p_i\cdot y_i <1\\
0&\textnormal{ otherwise}.
\end{array}$
 & $0$ & $0$\\
    L2-SVM* \cite{keerthi2005modified} &  $\frac{1}{2}\sum_{i=1}^\lsize\max(0, 1-p_i \cdot y_i)^2$ & 
$
\begin{array}{cc}
p_i-y_i &\textnormal{ if }p_i \cdot y_i <1\\
0&\textnormal{ otherwise}.
\end{array}$ &
$
\begin{array}{cc}
1 &\textnormal{ if } p_i\cdot y_i <1\\
0&\textnormal{ otherwise}.
\end{array}$ & $0$
 \\
    Logistic* \cite{walker1967estimation} & $\sum_{i=1}^\lsize\log(1+e^{-y_i p_i})$ & $- y_i(1+e^{y_i p_i})^{-1}$ & $e^{y_i p_i}(1+e^{y_i p_i})^{-2}$ & $0$ \\
RankRLS \cite{pahikkala2009efficient} & $\frac{1}{4}\sum_{i=1}^\lsize \sum_{j=1}^\lsize (y_i-p_i-y_j+p_j)^2$ & $\sum_{j=1}^\lsize(y_j-p_j)+n(p_i-y_i)$ & $n-1$ & $-1$ \\
\hline
	\end{tabular}
	\label{tb:losses}
\end{table*}

The regularized risk minimization problem provides a convex minimization problem, whose optimum can be located with (sub)gradient information. Next, we show how to efficiently compute the gradient of the objective function, and for twice differentiable loss functions Hessian-vector products, when using the Kronecker product kernel. For non-smooth losses or losses with non-smooth derivatives, we may instead consider subgradients and the generalized Hessian matrix (see e.g. \cite{keerthi2005modified,lin2008trust,teo2010bmr}). In this section, we will make use of the denominator-layout notation. 

First, let us consider making predictions on training data in the dual case. Based on previous considerations, we can compute 
\begin{equation}
\label{eqn:dualpredictions}
\bm{p}=\bm{R}(\bm{G}\otimes\bm{K})\bm{R}\transpose\bm{a}.
\end{equation}
The regularizer can be computed as $\frac{\lambda}{2}\Arrowvert f \Arrowvert_{\funspace}^2 = \frac{\lambda}{2}\bm{a}\transpose\bm{R}(\bm{G}\otimes\bm{K})\bm{R}\transpose\bm{a}$.

We use the following notation to denote the gradient and the Hessian (or a subgradient and/or generalized Hessian) of the loss function, with respect to $\bm{p}$:
\begin{equation*}
\bm{g}= \frac{\partial \lossfunction}{\partial \bm{p}} \textrm{ and }\bm{H}=\frac{\partial^2 \lossfunction}{\partial\bm{p}^2}
\end{equation*}
The exact form of $\bm{g}$ and $\bm{H}$ depends on the loss function, Table~\ref{tb:losses} contains some common choices in machine learning. While maintaining the full $\lsize\times\lsize$ Hessian would typically not be computationally feasible, for univariate losses the matrix is diagonal. Further, for many multivariate losses efficient algorithms for computing Hessian-vector products are known (see e.g. \cite{pahikkala2009efficient} for examples of efficiently decomposable ranking losses).

The gradient of the empirical loss can be decomposed via chain rule as 
\begin{equation*}
\frac{\partial \lossfunction}{\partial \bm{a}}=\frac{\partial \bm{p}}{\partial \bm{a}} \frac{\partial \lossfunction}{\partial \bm{p}} = \bm{R}(\bm{G}\otimes\bm{K})\bm{R}\transpose \bm{g}
\end{equation*}
The gradient of the regularizer can be computed as $\lambda\bm{R}(\bm{G}\otimes\bm{K})\bm{R}\transpose\bm{a}$. Thus, the gradient of the objective function becomes
\begin{equation}
\label{eqn:dualgradient}
\frac{\partial J}{\partial \bm{a}} = \bm{R}(\bm{G}\otimes\bm{K})\bm{R}\transpose(\bm{g} + \lambda\bm{a})
\end{equation}

The Hessian of $\lossfunction$ with respect to $\bm{a}$ can be defined as 
\begin{equation*}
\frac{\partial^2 \lossfunction}{\partial \bm{a}^2}=
\frac{\partial \bm{p}}{\partial \bm{a}} \frac{\partial}{\partial \bm{p}}  (\frac{\partial \bm{p}}{\partial \bm{a}} \frac{\partial \lossfunction}{\partial \bm{p}})
= \bm{R}(\bm{G}\otimes\bm{K})\bm{R}\transpose \bm{H} \bm{R}(\bm{G}\otimes\bm{K})\bm{R}\transpose
\end{equation*}
Hessian for the regularizer is defined as $\lambda\bm{R}(\bm{G}\otimes\bm{K})\bm{R}\transpose$. Thus the Hessian for the objective function becomes
\begin{equation*}
\frac{\partial^2 J}{\partial \bm{a}^2} = \bm{R}(\bm{G}\otimes\bm{K})\bm{R}\transpose(\bm{H}\bm{R}(\bm{G}\otimes\bm{K})\bm{R}\transpose+\lambda\idmatrix)
\end{equation*}

For commonly used univariate losses, assuming $\bm{p}$ has been computed for the current solution,
$\bm{g}$, and $\bm{H}$ can be computed in $O(\lsize)$ time. The overall cost of the loss, gradient and Hessian-vector product computations will then be dominated by the efficiency of the Kronecker-kernel vector product algorithm, resulting in $O(\qsize\lsize+\osize\lsize)$ time.

Conversely. in the primal case we can compute the predictions as $\bm{p}=\bm{R}(\bm{T}\otimes\bm{D})\bm{w}$, loss gradient as $(\bm{T}\transpose\otimes\bm{D}\transpose)\bm{R}\transpose\bm{g}$ and Hessian for the loss as $(\bm{T}\transpose\otimes\bm{D}\transpose)\bm{R}\transpose\bm{H}\bm{R}(\bm{T}\otimes\bm{D})$, (both w.r.t. $\bm{w}$). For the regularizer the corresponding values are $\frac{\lambda}{2}\bm{w}\transpose\bm{w}$, $\lambda\bm{w}$ and $\lambda\idmatrix$. The overall complexity of the loss, gradient and Hessian-vector product computations is $O(\min(\qsize\odsize\qdsize+\odsize\lsize,\osize\odsize\qdsize+\qdsize\lsize))$.

\subsection{Optimization framework}\label{sec:optimization}

\begin{algorithm}[t]
  \begin{algorithmic}[1]
    \Require $\bm{R}\in\{0,1\}^{\lsize\times\osize\qsize}$, $\bm{K}\in\mathbb{R}^{\osize\times\osize}$, $\bm{G}\in\mathbb{R}^{\qsize\times\qsize}$, $\bm{y}\in\mathbb{R}^{\lsize}$, $\lambda>0$
    \State $\dcoeffs\gets\bm{0}$
    \Repeat
    \State $\bm{p}\gets\bm{R}(\bm{G}\otimes\bm{K})\bm{R}\transpose\bm{a}$
    \State COMPUTE $\bm{H}$, $\bm{g}$
    \State SOLVE $(\bm{H}\bm{R}(\bm{G}\otimes\bm{K})\bm{R}\transpose+\lambda\idmatrix)\bm{x}=\bm{g} + \lambda\bm{a}$ \label{solvedual}
    \State $\bm{a}\gets\bm{a}-\delta\bm{x}$ ($\delta$ constant or found by line search)
  \Until convergence
  \caption{Dual Truncated Newton optimization for regularized risk minimization}
  \label{alg:truncnewtdual}
  \end{algorithmic}
\end{algorithm}

\begin{algorithm}[t]
  \begin{algorithmic}[1]
    \Require $\bm{R}\in\{0,1\}^{\lsize\times\osize\qsize}$, $\bm{D}\in\mathbb{R}^{\osize\times\odsize}$, $\bm{T}\in\mathbb{R}^{\qsize\times\qdsize}$, $\bm{y}\in\mathbb{R}^{\lsize}$, $\lambda>0$
    \State $\coeffs\gets\bm{0}$
    \Repeat
    \State $\bm{p}\gets\bm{R}(\bm{T}\otimes\bm{D})\coeffs$
    \State COMPUTE $\bm{H}$, $\bm{g}$
    \State \begin{varwidth}[t]{\linewidth} SOLVE $((\bm{T}\transpose\otimes\bm{D}\transpose)\bm{R}\transpose\bm{H}\bm{R}(\bm{T}\otimes\bm{D})+\lambda\idmatrix)\bm{x} =$ \\\hskip\algorithmicindent $(\bm{T}\transpose\otimes\bm{D}\transpose)\bm{R}\transpose\bm{g}+\lambda\coeffs$.
\end{varwidth}\label{solveprimal}
    \State $\bm{w}\gets\bm{w}-\delta\bm{x}$ ($\delta$ constant or found by line search)
  \Until convergence
  \caption{Primal Truncated Newton optimization for regularized risk minimization}
  \label{alg:truncnewtprimal}
  \end{algorithmic}
\end{algorithm}

As a simple approach that can be used for training regularized risk minimization methods with access to gradient and (generalized) Hessian-vector product operations, we consider a truncated Newton optimization scheme (Algorithms $\ref{alg:truncnewtdual}$ and $\ref{alg:truncnewtprimal}$; these implement approaches similar to \cite{keerthi2005modified,Chapelle2007primal}). This is a second order optimization method that on each iteration computes the Newton step direction $\partial^2 J(f)\bm{x}=\partial J(f)$ approximately up to a pre-defined number of steps for the linear system solver. The approach is well suited for Kronecker product kernel method optimization, since while computing the Hessians explicitly is often not computationally feasible, computing Hessian-vector products can be done efficiently using the generalized vec trick algorithm. 

Alternative optimization schemes can certainly be used, such as the Limited-memory BFGS algorithm \cite{byrd1995limited} or trust-region Newton optimization \cite{lin2008trust}. Further, for non-smooth losses such as the hinge loss, methods tailored for non-smooth optimization such as the bundle method \cite{teo2010bmr} should be used rather than the truncated Newton method. Such works are orthogonal to our work, as the proposed algorithm can be used to speed up computations for any optimization method that is based on (sub)gradient, and possibly (generalized) Hessian-vector product computations. Optimization methods that process either the edges or the model coefficients individually or in small batches (e.g. stochastic gradient descent \cite{bottou-2010}, coordinate descent \cite{chang2008coordinate}, SVM decomposition methods \cite{fan2005working}) may however not be a good match for the proposed algorithm, as the largest speed-up is gained when doing the computations for all of the training data in one single batch.

Each iteration of the Truncated Newton algorithm starts by computing the vector of training set predictions $\bm{p}$ for the current solution, after which
the gradient $\bm{g}$ and Hessian $\bm{H}$ can be computed (see Table~\ref{tb:losses} for examples on how to compute these
for several widely used loss functions). After this, the algorithm
solves approximately the linear system
\begin{equation*}
\frac{\partial^2 J}{\partial \bm{a}^2}\bm{x}=\frac{\partial J}{\partial \bm{a}},
\end{equation*}
to find the next direction of descent (for the primal case we solve analogously $\frac{\partial^2 J}{\partial \bm{w}^2}\bm{x}=\frac{\partial J}{\partial \bm{w}}$). Here, we may use methods developed for solving linear systems, such as the 
Quasi-Minimal Residual iteration method \cite{freund1991qmr} used in our experiments. For the dual case, we may simplify this system of equations by removing the common term $\bm{R}(\bm{G}\otimes\bm{K})\bm{R}\transpose$ from both sides, resulting in the system
\begin{equation}
\label{eqn:simpledualsystem}
(\bm{H}\bm{R}(\bm{G}\otimes\bm{K})\bm{R}\transpose+\lambda\idmatrix)\bm{x}=\bm{g} + \lambda\bm{a}\;.
\end{equation}
While the optimization approaches can require a large number of iterations in order to converge, in practice often good predictive accuracy can be obtained using early stopping both in solving the system of linear equations on line~\ref{solvedual} of Algorithm~\ref{alg:truncnewtdual} and line~\ref{solveprimal} of Algorithm~\ref{alg:truncnewtprimal} and the outer truncated Newton optimization loops, as there is no need to continue optimization once the error of the prediction function stops decreasing on a separate validation set (see e.g. \cite{yao2007early,gerfo2008spectral,airola2010largescale}). 


\subsection{Complexity comparison}

The operations that dominate the computational cost for Algorithms \ref{alg:truncnewtdual} and \ref{alg:truncnewtprimal} are matrix-vector products of the form $\bm{R}(\bm{G}\otimes\bm{K})\bm{R}\transpose\bm{v}$ and $(\bm{T}\transpose\otimes\bm{D}\transpose)\bm{R}\transpose\bm{v}$. A natural question to ask is, how much does the proposed fast generalized vec Trick algorithm (Algorithm~\ref{alg:bgkbvdetailed}) speed these operations up, compared to explicitly forming the Kronecker product matrices (here denoted the 'Baseline' method).

We consider three different settings:
\begin{itemize}
\item Independent: all edges in the training set are mutually disjoint, indicating that $\lsize=\osize=\qsize$.
\item Dependent: training edges may share start or end vertices or both with each other, that is $\max(\osize,\qsize)<\lsize<\osize\qsize$.
\item Complete: the training set is a complete bipartite graph, that is $\lsize=\osize\qsize$.
\end{itemize}
'Dependent' is the default setting assumed in this paper, while 'Independent' and 'Complete' are extreme cases of the setting.

In Table~\ref{tb:dual_comparison} we provide the comparison for the dual case. In the 'Independent' -case the complexity of the proposed algorithm reduces to that
of the baseline method, since there are vertices shared between the edges, that the algorithm could make use of. However, beyond that the complexity of
the baseline approach grows quadratically with respect to the number of edges, while the complexity of the proposed method grows linearly with respect to the
numbers of start vertices, end vertices, and edges. In Table~\ref{tb:primal_comparison} we provide the comparison for the primal case. Again, the complexity is the same
for the 'Independent' case, but the proposed method is much more efficient for the other settings assuming $\osize << \lsize$ and $\qsize<<\lsize$.
Finally, we note that  for both the primal and the dual settings, for the 'Complete' case where $\bm{R}=\bm{I}$, the multiplication could as efficiently be computed using the 'VecTrick' - method implied by Roth's column lemma (\ref{veclemma}).
To conclude, we observe that the proposed algorithm considerably outperforms previously known methods, if $\max(\osize,\qsize)<<\lsize<\osize\qsize$.

\begin{table}[t]
\caption{Dual case: complexity comparison of the proposed method and the baseline approach that constructs the Kronecker product kernel matrix explicitly.}
\centering
\begin{tabular}{|l|ll|}
\hline
& Baseline & Proposed\\
\hline
Independent & $O(\lsize^2)$ & $O(\lsize^2)$\\
Dependent & $O(\lsize^2)$ & $O(\qsize\lsize+\osize\lsize)$\\
Complete & $O(\osize^2\qsize^2)$ & $O(\osize^2\qsize+\osize\qsize^2)$\\
\hline
\end{tabular}
\label{tb:dual_comparison}
\end{table}

\begin{table}[t]
\caption{Primal case: complexity comparison of the proposed method and the baseline approach that constructs the Kronecker data matrix explicitly.}
\centering
\begin{tabular}{|l|ll|}
\hline
& Baseline & Proposed\\
\hline
Independent & $O(\lsize \odsize\qdsize)$ & $O(\lsize \odsize\qdsize)$\\
Dependent & $O(\lsize \odsize\qdsize)$  & $O(\min(\qsize\odsize\qdsize+\odsize\lsize,\osize\odsize\qdsize+\qdsize\lsize))$\\
Complete & $O(\osize\qsize\odsize\qdsize)$  & $O(\min(\osize\odsize\qdsize+\osize\qsize\qdsize,\odsize\qdsize\qsize+\odsize\osize\qsize))$\\
\hline
\end{tabular}
\label{tb:primal_comparison}
\end{table}

\section{Case studies: ridge regression and support vector machines}

Next, as an example on how our framework may be applied we consider two commonly used losses, deriving fast ridge regression and support vector machine training algorithms for the Kronecker product kernel.

\subsection{Kronecker ridge regression}

As our first example learning algorithm that uses the above defined concepts, let us consider the well-known ridge regression method (aka regularized least-squares, aka least-squares SVM) \cite{hoerl1970ridge,poggio2003mathematics}. For this specific case, the
resulting algorithm is simpler than the general optimization framework of Algorithms~\ref{alg:truncnewtdual} and~\ref{alg:truncnewtprimal}, as the linear system appearing in the algorithms needs to be solved only once.

In this case, $\lossfunction=\frac{1}{2}(\bm{p}-\bm{y})\transpose(\bm{p}-\bm{y})$, $\bm{H}=\bm{I}$ and $\bm{g} = \bm{p}-\bm{y}$ (see Table~\ref{tb:losses}). For the dual case, based on (\ref{eqn:dualgradient}) it can be observed that the full gradient for the ridge regression is
\begin{equation*}
\bm{R}(\bm{G}\otimes\bm{K})\bm{R}\transpose (\bm{p}-\bm{y} + \lambda\bm{a})
\end{equation*}
Writing $\bm{p}$ open (\ref{eqn:dualpredictions}) and setting the gradient to zero, we recover a linear system
\begin{equation*}
\bm{R}(\bm{G}\otimes\bm{K})\bm{R}\transpose(\bm{R}(\bm{G}\otimes\bm{K})\bm{R}\transpose+ \lambda\bm{I})\bm{a} = \bm{R}(\bm{G}\otimes\bm{K})\bm{R}\transpose\bm{y}
\end{equation*}
A solution to this system can be recovered by canceling out $\bm{R}(\bm{G}\otimes\bm{K})\bm{R}\transpose$ from both sides, resulting in the linear system
\begin{equation*}
(\bm{R}(\bm{G}\otimes\bm{K})\bm{R}\transpose+ \lambda\bm{I})\bm{a} = \bm{y}
\end{equation*}
This system can be solved directly via standard iterative solvers for systems of linear equations. Combined with the generalized vec trick algorithm, this results in $O(\osize\lsize+\qsize\lsize)$ cost for each iteration of the method.

For the primal case, the gradient can be expressed as 
\begin{equation*}
(\bm{T}\transpose\otimes\bm{D}\transpose)\bm{R}\transpose(\bm{p}-\bm{y}) + \lambda \bm{w}
\end{equation*}
Again, writing $\bm{p}$ open and setting the gradient to zero, we recover a linear system
\begin{equation*}
(\bm{T}\transpose\otimes\bm{D}\transpose)\bm{R}\transpose\bm{R}(\bm{T}\otimes\bm{D})+\lambda\bm{I})\bm{w} = (\bm{T}\transpose\otimes\bm{D}\transpose)\bm{R}\transpose\bm{y}
\end{equation*}
Solving this system with a linear system solver together with the generalized Kronecker product algorithm, each iteration will require $O(\min(\osize\odsize\qdsize+\lsize\qdsize,\odsize\qdsize\qsize+\odsize\lsize))$ cost.

The naive approach of training ridge regression with the explicitly formed Kronecker product kernel or data matrix using standard solvers of systems of linear equations would result in $O(\lsize^2)$ and $O(\lsize\odsize\qdsize)$ cost per iteration for the dual and primal cases, respectively.

\subsection{Kronecker support vector machine}

Support vector machine is one of the most popular classification methods in machine learning. Two of the most widely used variants of the method are the so-called L1-SVM and L2-SVM (see Table~\ref{tb:losses}). Following works such as \cite{keerthi2005modified,Chapelle2007primal,lin2008trust} we consider the L2-SVM variant, since unlike L1-SVM it has an objective function that is differentiable and yields a non-zero generalized Hessian matrix, making it compatible with the presented Truncated Newton optimization framework.

The loss can be defined as $\lossfunction=\frac{1}{2}\sum_{i=1}^\lsize\max(0, 1-p_i \cdot y_i)^2$, where $y\in\{-1,1\}$. Let $\svecs=\{ i |y\cdot p(\bm{x}_i)< 1, i\in[n]\}$ denote the subset of training set that have non-zero loss for given prediction function. Further, let $\bm{S}_{\svecs}$ denote the index matrix corresponding to this index set. Now we can re-write the loss in a least-squares form as $\lossfunction=\frac{1}{2}(\bm{S}_{\svecs}\bm{p}-\bm{S}_{\svecs}\bm{y})\transpose(\bm{S}_{\svecs}\bm{p}-\bm{S}_{\svecs}\bm{y})$, its gradient as $\bm{g}=\bm{S}_{\svecs}\transpose(\bm{S}_{\svecs}\bm{p}-\bm{S}_{\svecs}\bm{y})$, and the Hessian as $\bm{H}=\bm{S}_{\svecs}\transpose\bm{S}_{\svecs}$, which is a diagonal matrix with entry $1$ for all members of $\svecs$, and $0$ otherwise.

By inserting $\bm{g}$ to (\ref{eqn:dualgradient}) and writing $\bm{p}$ open (\ref{eqn:dualpredictions}) we recover the gradient of the L2-SVM objective function, with respect to $\dcoeffs$, as
\begin{align}
\bm{R}(\bm{G}\otimes\bm{K})\bm{R}\transpose
(\bm{S}_{\svecs}\transpose(\bm{S}_{\svecs}\bm{R}(\bm{G}\otimes\bm{K})\bm{R}\transpose\dcoeffs-\bm{S}_{\svecs}\bm{y})
+\lambda\dcoeffs)
\end{align}
Inserting $\bm{g}$ and $\bm{H}$ to (\ref{eqn:simpledualsystem}), we see that $\frac{\partial J}{\partial \bm{a}\partial \bm{a}}\bm{x}=\frac{\partial J}{\partial \bm{a}}$ can be solved from:
\begin{equation*}
\label{dualsvmnewton}
(\bm{S}_{\svecs}\transpose\bm{S}_{\svecs}\bm{R}(\bm{G}\otimes\bm{K})\bm{R}\transpose+\lambda\idmatrix)\bm{x}=(\bm{S}_{\svecs}\transpose(\bm{S}_{\svecs}\bm{R}(\bm{G}\otimes\bm{K})\bm{R}\transpose\dcoeffs- \bm{S}_{\svecs}\bm{y})
+\lambda\dcoeffs)
\end{equation*}
The gradient and Hessian-vector products can be computed again at $O(\qsize\lsize+\osize\lsize)$ time. However, it should be noted that the support vector machine algorithm encourages sparsity, meaning that as the optimization progresses $\bm{a}$ may come to contain many zero coefficients, and the set ${\svecs}$ may shrink (at the optimal solution these two sets coincide, denoting the so-called support vectors). Thus given a new solution $\bm{a}$, we may compute $\bm{p}=\bm{R}(\bm{G}\otimes\bm{K})\bm{R}\transpose\dcoeffs$ and thus also the right-hand size of (\ref{dualsvmnewton}) in 
$O(\min(\qsize \Arrowvert \bm{a} \Arrowvert_{0}+\osize\lsize, \osize\Arrowvert \bm{a} \Arrowvert_{0} + \qsize\lsize)$ time, by removing the zero-coefficients from $\bm{a}$ and the corresponding columns from $\bm{R}^T$. Further, when solving (\ref{dualsvmnewton}) with Truncated Newton optimization, we can in each inner iteration of the method compute matrix-vector multiplication $\bm{S}_{\svecs}\bm{R}(\bm{G}\otimes\bm{K})\bm{R}\transpose\bm{v}$ in $O(\min(\qsize \Arrowvert \bm{v} \Arrowvert_{0}+\osize\arrowvert \svecs \arrowvert, \osize\Arrowvert \bm{v} \Arrowvert_{0} + \qsize\arrowvert \svecs \arrowvert)$ time.

In the primal case the gradient can be written as 
\begin{equation*}
(\bm{T}\transpose\otimes\bm{D}\transpose)\bm{R}\transpose\bm{S}_{\svecs}\transpose(\bm{S}_{\svecs}\bm{p}-\bm{S}_{\svecs}\bm{y})+\lambda\coeffs
\end{equation*}
A generalized Hessian matrix can be defined as
\begin{equation*}
\bm{H} = (\bm{T}\transpose\otimes\bm{D}\transpose)\bm{R}\transpose\bm{S}_{\svecs}\transpose\bm{S}_{\svecs}\bm{R}(\bm{T}\otimes\bm{D})+\lambda\idmatrix
\end{equation*}
The most efficient existing SVM solvers can be expected at best to scale quadratically with respect to training set size when solving the dual problem,
and linearly with respect to training set size and number of features when solving the primal case \cite{bottou2007svmsolvers,joachims2006training}.
Thus using the full Kronecker product kernel or data matrices, such solvers would have $O(\lsize^2)$ and $O(\lsize\odsize\qdsize)$ scaling for the dual and primal
cases, respectively.

\section{Experiments}

Computational costs for iterative Kronecker product kernel method training
depend on two factors: the cost of a single iteration, and
the number of iterations needed to reach a good solution.
The cost of a single iteration is dominated by 
gradient computations and Hessian-vector products. These can be efficiently
performed with Algorithm~\ref{alg:bgkbvdetailed}.
Further, as discussed at the end of Section~\ref{sec:optimization}, the number of
iterations needed can be limited via early stopping once a predictor that works well
on independent validation data has been reached.
For prediction times, the
dominating costs are the data matrix or kernel matrix multiplications
with the primal or (possibly sparse) dual coefficient vectors, both operations
that can be speeded up with Algorithm~\ref{alg:bgkbvdetailed}.

In the experiments, we demonstrate the following:
\begin{enumerate}
\item Parameter selection: we show how the hyperparameters of the methods
can be tuned using validation data.
\item Fast training: the combination of the proposed short-cuts and early stopping
allows orders of magnitude faster training than with regular kernel method solvers.
\item Fast prediction: proposed short-cuts allow orders of magnitude faster
prediction for new data, than with regular kernel predictors.
\item Accurate predictions: the predictive accuracy of Kronecker models is competitive
compared to alternative types of scalable graph learning methods.
\end{enumerate}

We implement the Kronecker methods in Python, they are made freely available under open source
license as part of the RLScore machine learning library\footnote{\url{https://github.com/aatapa/RLScore}} \cite{pahikkala2016rlscore}.
For comparison, we consider the LibSVM software \cite{chang2011libsvm}, which implements a highly efficient
support vector machine training algorithm \cite{fan2005working}. Further, we compare our approach to stochastic gradient descent and K-nearest neighbor based graph prediction methods.

\subsection{Data and general setup}

\begin{table}[H]
\centering
	\caption{Data sets}
\setlength{\tabcolsep}{3pt}
	\begin{tabular}{ l | c c c c c}
	   Data set & edges & pos. & neg. & start vertices & end vertices \\
       \hline 
        Ki &  93356 & 3200 & 90156 & 1421 & 156  \\
        GPCR & 5296 & 165 & 5131  & 223 & 95   \\
        IC & 10710 & 369 & 10341  & 210 & 204   \\
        E & 73870 & 732 & 73138 & 445 & 664  \\
        Checker & 250000 & 125000 & 12500 & 1000 & 1000\\
        Checker+ & 10240000 & 5120000 & 5120000 & 6400 & 6400\\
        \hline
 	\end{tabular}
	\label{tb:datasets}
\end{table}

As a typical bipartite graph learning problem, we consider the problem of
predicting drug-target interactions. From a data base of known drugs, targets and their
binary interactions, the goal is to learn a model that can for new previously unseen drugs
and targets predict whether they interact. 
We consider four drug-target interaction data sets, the GPCR, IC, E data sets \cite{yamanishi2008prediction},
and the Ki data \cite{Metz2011kinome}. The Ki-data is a
naturally sparse data set where the class information is available only for a subset of the edges. For the other
four drug-target data sets we use for the experiments a subset of all the possible drug-target interactions.
We use exactly the same preprocessing of the data as \cite{pahikkala2015dti}. 
The characteristics of the data sets are described in
Table~\ref{tb:datasets}. The edges in the data are $(\bm{d}_i,\bm{t}_j,y_h)$ triplets, consisting of a pair of vectors
encoding the features of the drug and the target (see \cite{pahikkala2015dti} for details of the features), and a
label having value $1$ if the drug and target interact, and $-1$ if they do not. Mapping the
problem to the terminology used in this paper, we may consider the drugs as start vertices, the targets as end vertices, and
the drug-target pairs with known interaction or non-interaction as edges.

Further, we generated two data sets using a variant of the Checkerboard simulation, a standard non-linear problem for benchmarking large scale SVM solvers (see e.g. \cite{mangasarian2001lagrangian}). In our simulation both start and end vertices have a single feature describing them, drawn from continuous uniform distribution in range $(0,100)$. The output assigned
to an edge $(d,t)$ is +1 whenever both $\left \lfloor{d}\right \rfloor$ and
$\left \lfloor{t}\right \rfloor$ are either odd or even, and -1 when one of them is odd and the other even.
When plotting the output against $d$ and $t$, this results in a highly non-linear checkerboard-type of pattern. Finally, random
noise is introduced to the data by flipping with $0.2$ probability the class of each edge. The numbers of start and end vertices
are same, and labels are assigned for $25\%$ of all the possible edges (i.e. $\osize=\qsize$ and $\lsize = 0.25\osize^2$).

In the zero-shot learning setting the aim is to generalize to such $(\bm{d}_i,\bm{t}_j)$-edges that are vertex disjoint with the training graph. Therefore, cross-validation with graph data is more complicated than for standard i.i.d. data, since this aim must be reflected in the train-test split (see e.g. \cite{park2012flaws,pahikkala2015dti}). The split is illustrated in Figure~\ref{fig:cv}. To ensure that the training and test graphs are vertex disjoint, the edges are divided into training and test edges as follows. First, both the start vertex-indices $[1,...,\osize]$ and end vertex-indices $[1,...,\qsize]$ are randomly divided into a training and test part. Then, an edge $(\bm{d}_i,\bm{t}_j,y_h)$ is assigned to training set if $i$ belongs to the training start vertex indices and $j$ belongs to the training end vertex indices. It is assigned to the test set if $i$ belongs to the test start vertex indices and $j$ belongs to test end vertex indices. Finally, the rest of the edges are discarded, that is, they belong neither to the training nor test part (the greyed out blocks in Figure~\ref{fig:cv}).
To tune the hyper-parameters without the risk of overfitting, one can split the data into training, validation and test parts in an analogous way to the above described train-test split. Combining this with cross-validation is in the literature known as nested cross-validation (for a detailed description of this approach for graph learning, see \cite{pahikkala2015dti}).

\begin{figure}
\begin{center}
\includegraphics[width=0.35\linewidth]{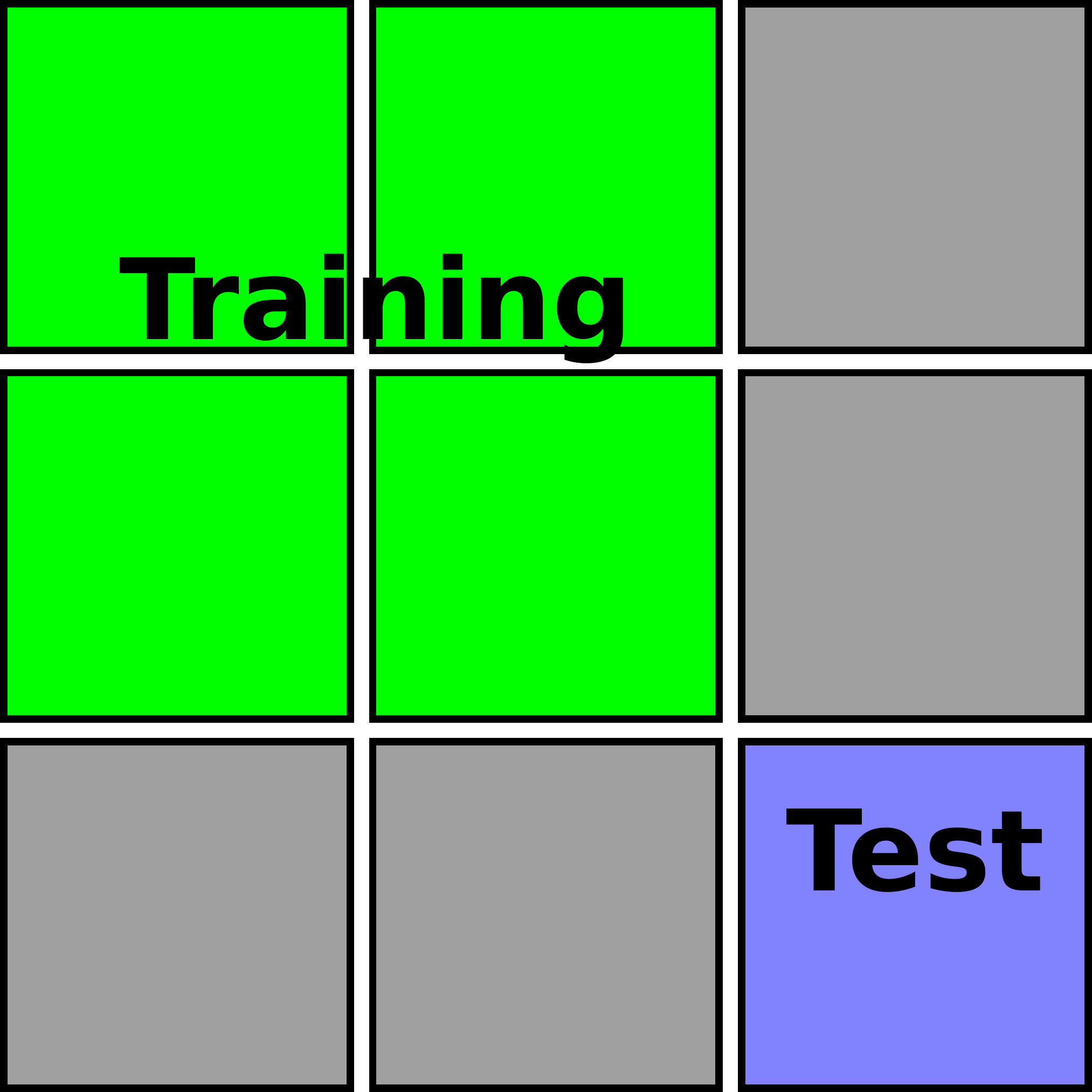}
\end{center}
\caption{Ninefold cross-validation. The matrix rows represent start vertices and columns end vertices,
the edges with known labels corresponds to a subset of the elements of this matrix.
The row- and column indices are both divided into three sets. The Cartesian product sets of the row and column sets index nine blocks of possible edges. On each round, the test
fold is formed from all the edges that belong to one of the blocks, and further have also known label.
The training folds are formed from the four blocks, that share no common rows or columns with the test fold.
Four blocks are left unused, as their edges connect vertices belonging to training and test folds.}
\label{fig:cv}
\end{figure}

Ridge regression is trained with the minimum residual iteration algorithm \cite{paige1975solution} implemented in the
scipy.sparse.linalg.minres package, while the inner optimization loop of the SVM training algorithm uses the scipy.sparse.linalg.qmr
implementing of quasi-minimal residual iteration \cite{freund1991qmr} (SciPy version 0.14.1), with $\delta=1$. 
Regularization parameters from the grid $[2^{-20},..., 2^{20}]$ were tested.  We restrict our plots  to values $[2^{-10}, 2^{-5}, 2^{0}, 2^{5}, 2^{10}]$, as these allow representing all the main
trends in the experiments. The classification performance on test data was measured with area under ROC curve (AUC).
The experiments were run on a desktop computer with Intel Core i7-3770 CPU (3.40GHz) running
Ubuntu Linux 15.04.

Early stopping experiments were run with the linear kernel, in order to allow comparing dual and primal optimization. The LibSVM comparison was done using the Gaussian kernel.
LibSVM does not directly support the Kronecker product kernel, this issue was resolved as follows.
If both start vertex and end vertex kernels are Gaussian with width $\gamma$, then
$\kernelfk(\bm{d}, \bm{d}')\kernelfk(\bm{t},\bm{t}')=e^{-\gamma\Arrowvert \bm{t}- \bm{t}'\Arrowvert^2}
e^{-\gamma\Arrowvert \bm{d}- \bm{d}'\Arrowvert^2}=e^{-\gamma\Arrowvert [\bm{t},\bm{d}]\transpose - [\bm{t'},\bm{d'}]\transpose \Arrowvert^2}$,
that is, the Kronecker product kernel is equal to using the Gaussian kernel with concatenated features of the start and end vertex.

\subsection{Choosing hyperparameters}

In this section, we set up guidelines for selecting the values of the hyperparameters based on experimental verification. The Kronecker algorithms have the following hyperparameters: the regularization parameter, the number of iterations for ridge regression, and both the number of inner and outer iterations for the SVM. We consider only the dual optimization,  the
observed behavior was very similar also for the primal case.


We run the optimization up to 100 iterations measuring regularized risk $J(f)$, and AUC on the test set
(in selected experiments with 500 iterations we found little improvements).
Ridge regression results are presented in Figure~\ref{fig:ridge_loss_perf}, while SVM results with inner
optimization loop terminated after 10 and 100 iterations are presented in Figures \ref{fig:svm10_loss_perf} and \ref{fig:svm100_loss_perf}.

In all the experiments the optimal test set AUC is obtained within tens
of iterations. Beyond this point reduction in regularized risk no longer translates into better predictions. Further, for
SVMs increasing the number of inner iterations from 10 to 100 allows achieving much faster decrease in regularized risk.
However, this comes at the cost of having to perform ten times more computation each iteration, and does not lead into
faster increase in test set AUC. Thus it can be observed, that rather than having to solve exactly the large optimization
problems corresponding to training Kronecker product kernel methods, often only a handful of iterations need to be performed in
order to obtain maximal predictive accuracy.

Several observations can be made based on the experiments. The regularized risk decreases quite quickly even if the SOLVE operation in Algorithms $\ref{alg:truncnewtdual}$ and $\ref{alg:truncnewtprimal}$ is terminated after a small number of iterations.
Moreover, early termination of SOLVE provides us more fine-grained control of the degree of fitting
(contrast Figures \ref{fig:svm10_loss_perf} and \ref{fig:svm100_loss_perf}, where for 100 inner iterations test performance in some cases starts immediately decreasing).

To conclude, the results suggest that on large data sets a good strategy is to start by setting the number of iterations to a small constant
(e.g. 10 inner and outer iterations), and increase these parameters only if the predictive accuracy
keeps increasing beyond the first iterations on independent validation set. Furthermore, we note that one could sidestep the selection of
regularization parameter $\lambda$ by setting it to a small constant and regularizing only with early stopping, however tuning
also $\lambda$ on separate validation data can sometimes yield even better performance.

\begin{figure}
\begin{center}
 \includegraphics[width=\columnwidth]{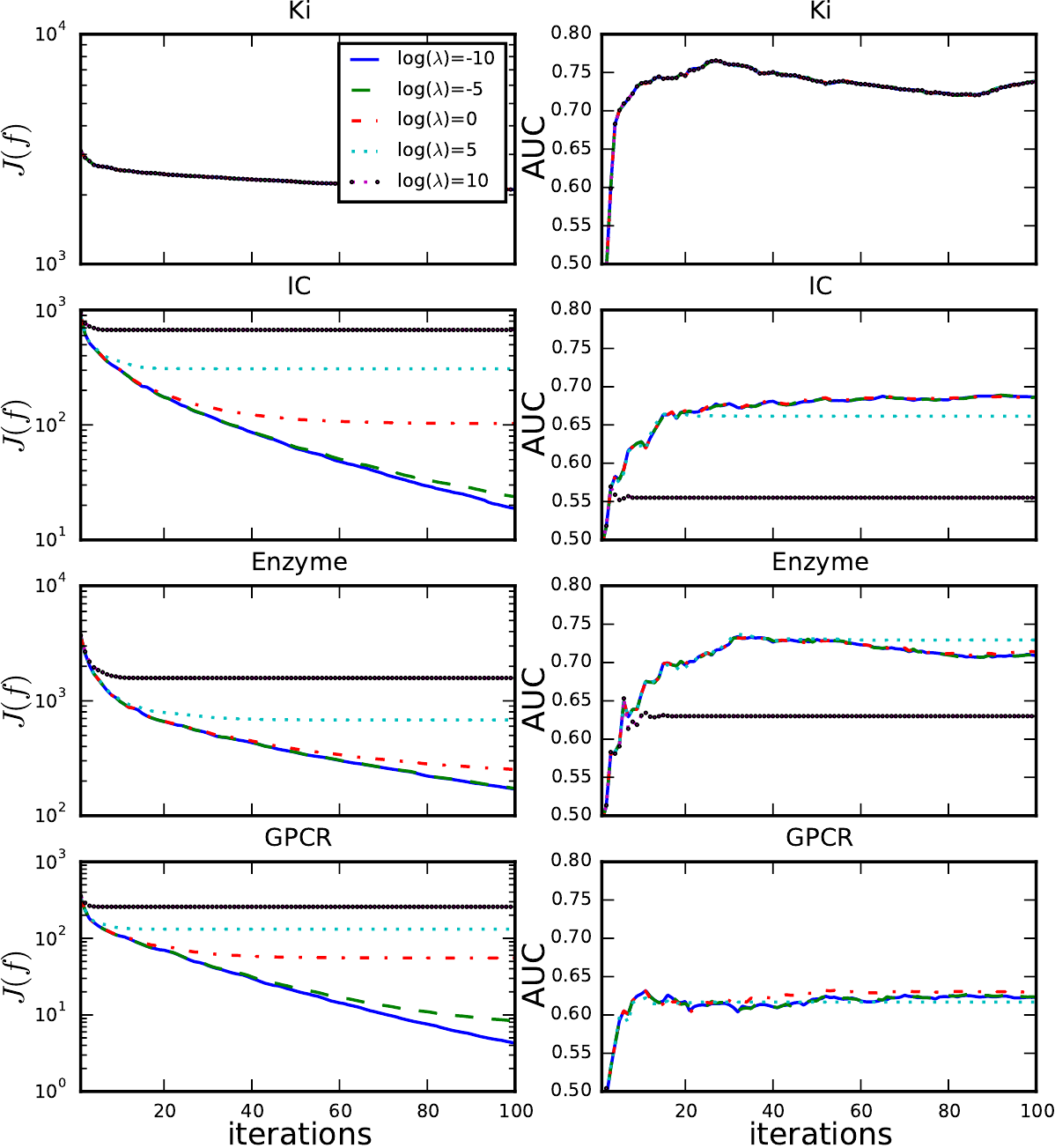}
 \caption{Ridge regression regularized risk (left) and test set AUC (right) as a function of optimization iterations.}
\label{fig:ridge_loss_perf}
\end{center}
\end{figure}

\begin{figure}
\begin{center}
 \includegraphics[width=\columnwidth]{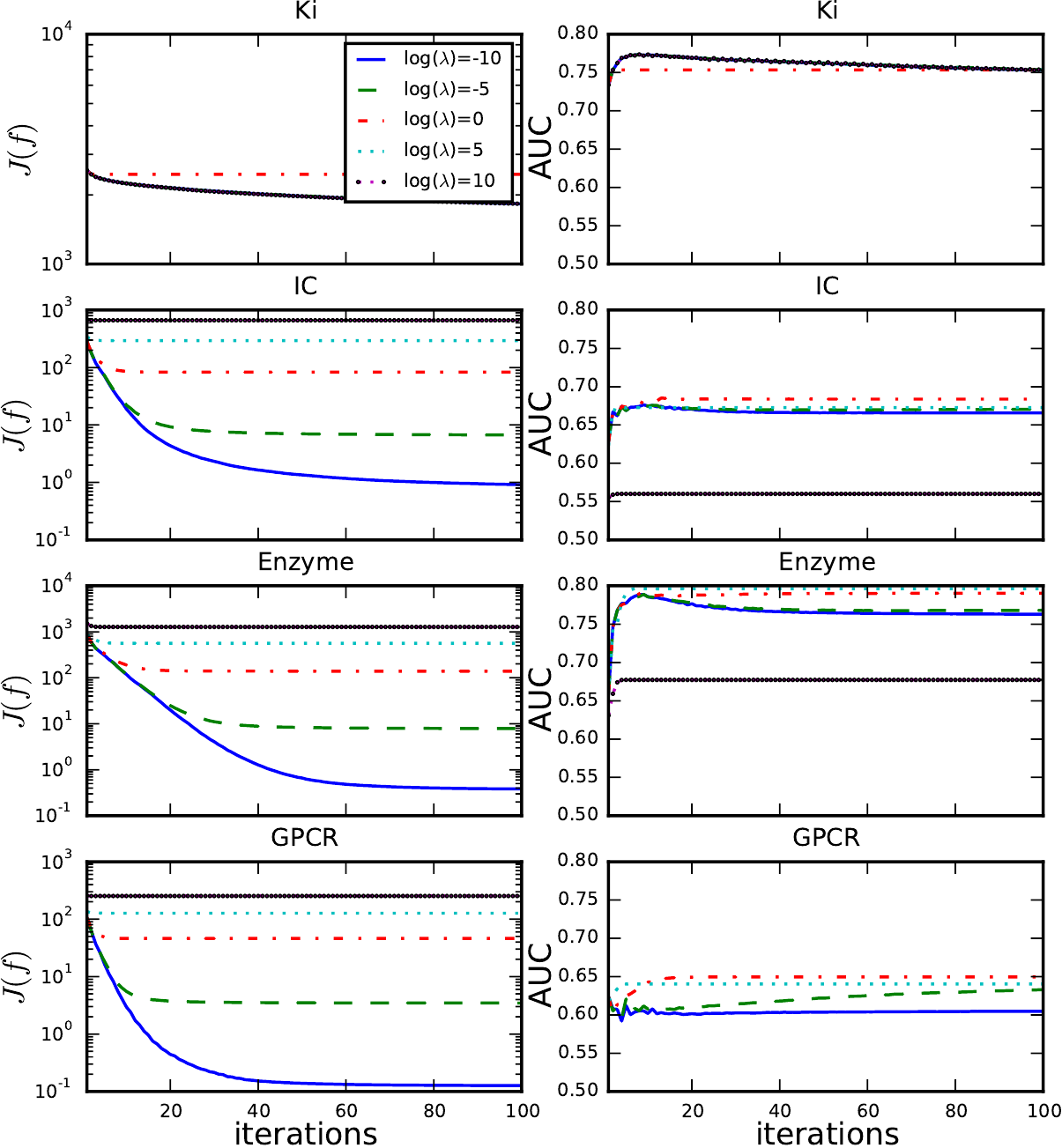}
 \caption{SVM with 10 inner iterations. Regularized risk (left) and test set AUC (right) as a function of outer optimization iterations.}
\label{fig:svm10_loss_perf}
\end{center}
\end{figure}

\begin{figure}
\begin{center}
 \includegraphics[width=\columnwidth]{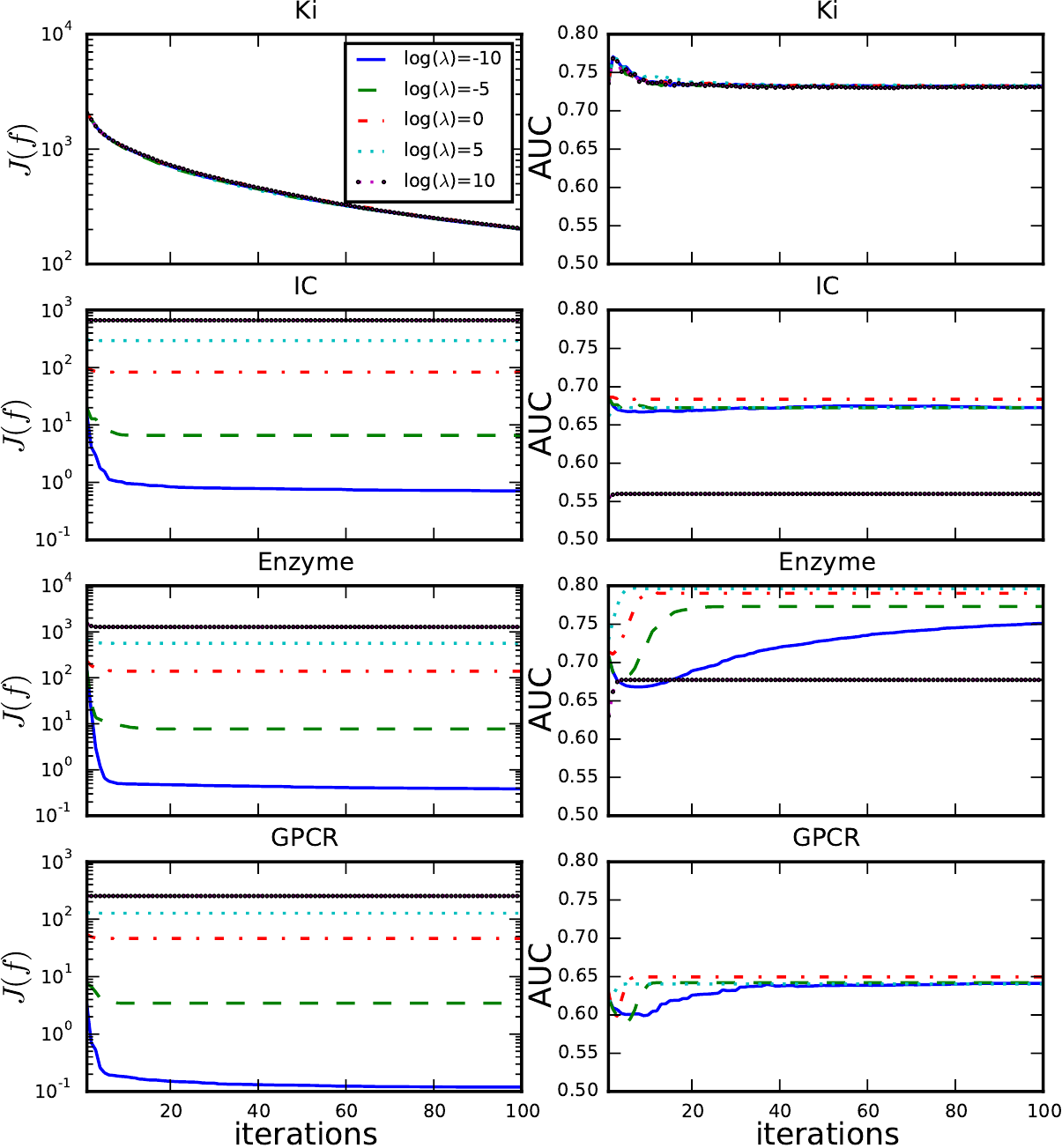}
 \caption{SVM with 100 inner iterations. Regularized risk (left) and test set AUC (right) as a function of outer optimization iterations.}
\label{fig:svm100_loss_perf}
\end{center}
\end{figure}

\begin{figure*}[t]
\begin{center}
\includegraphics[width=0.26\paperwidth]{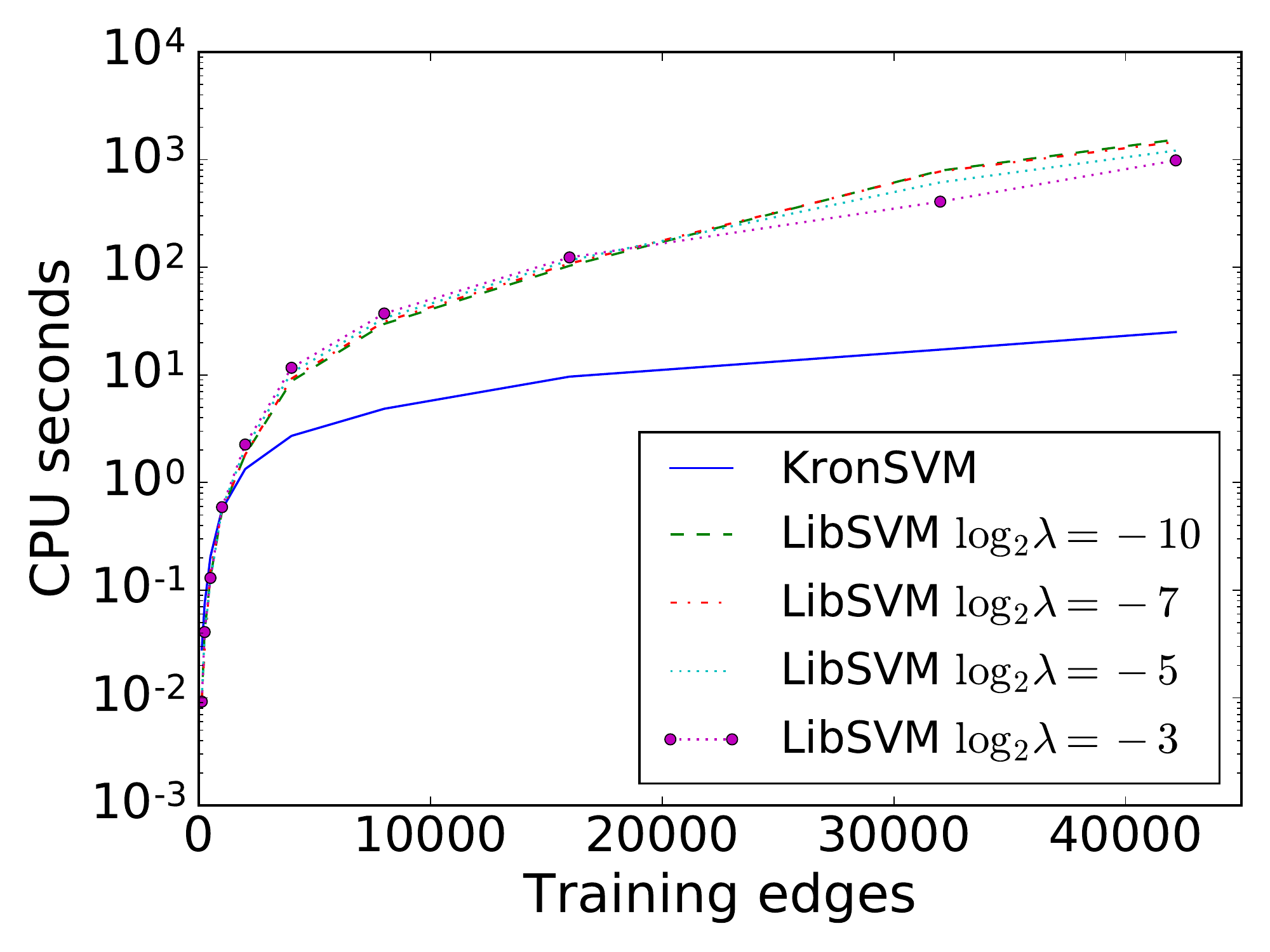}
\includegraphics[width=0.26\paperwidth]{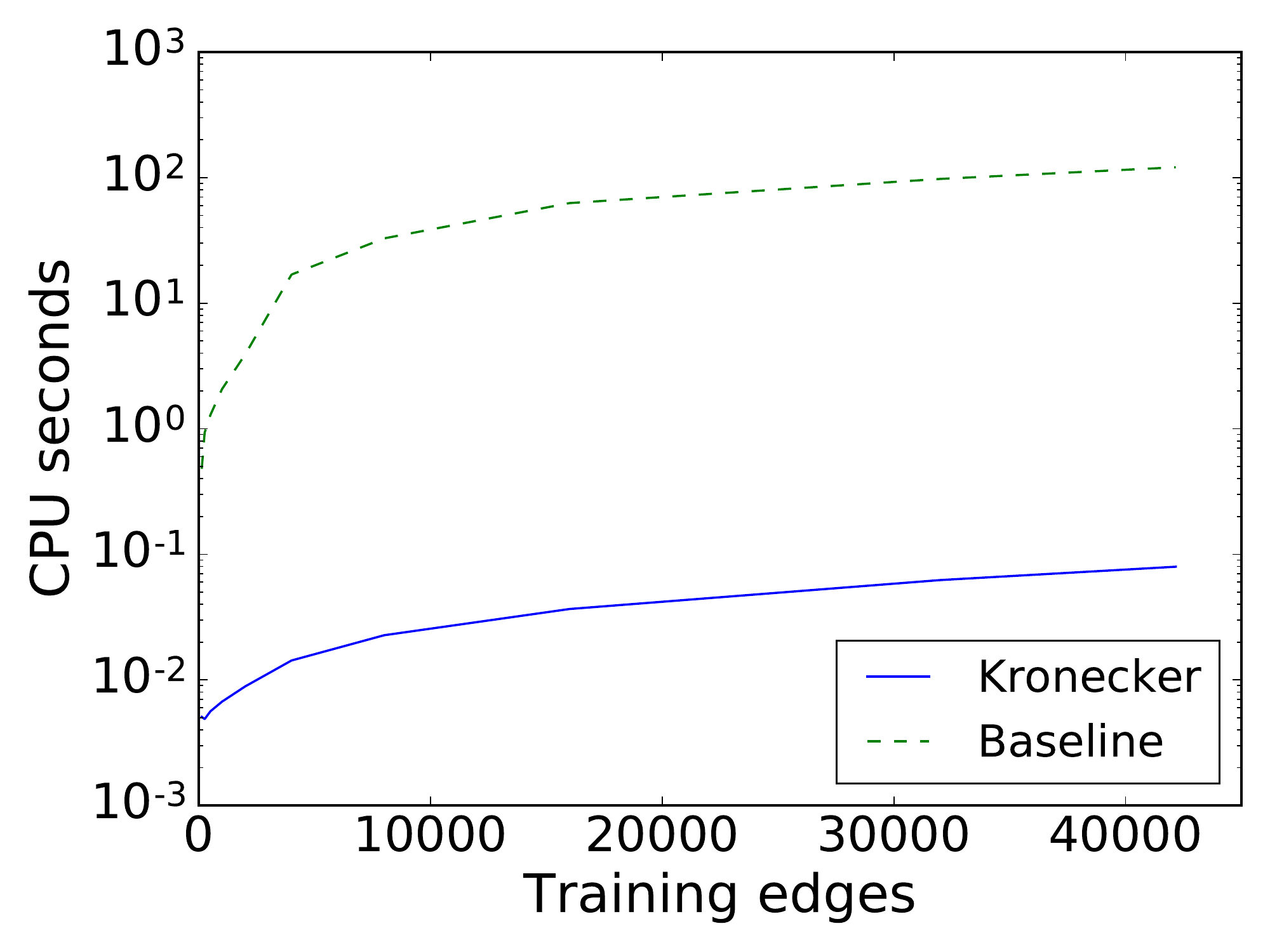}
\includegraphics[width=0.26\paperwidth]{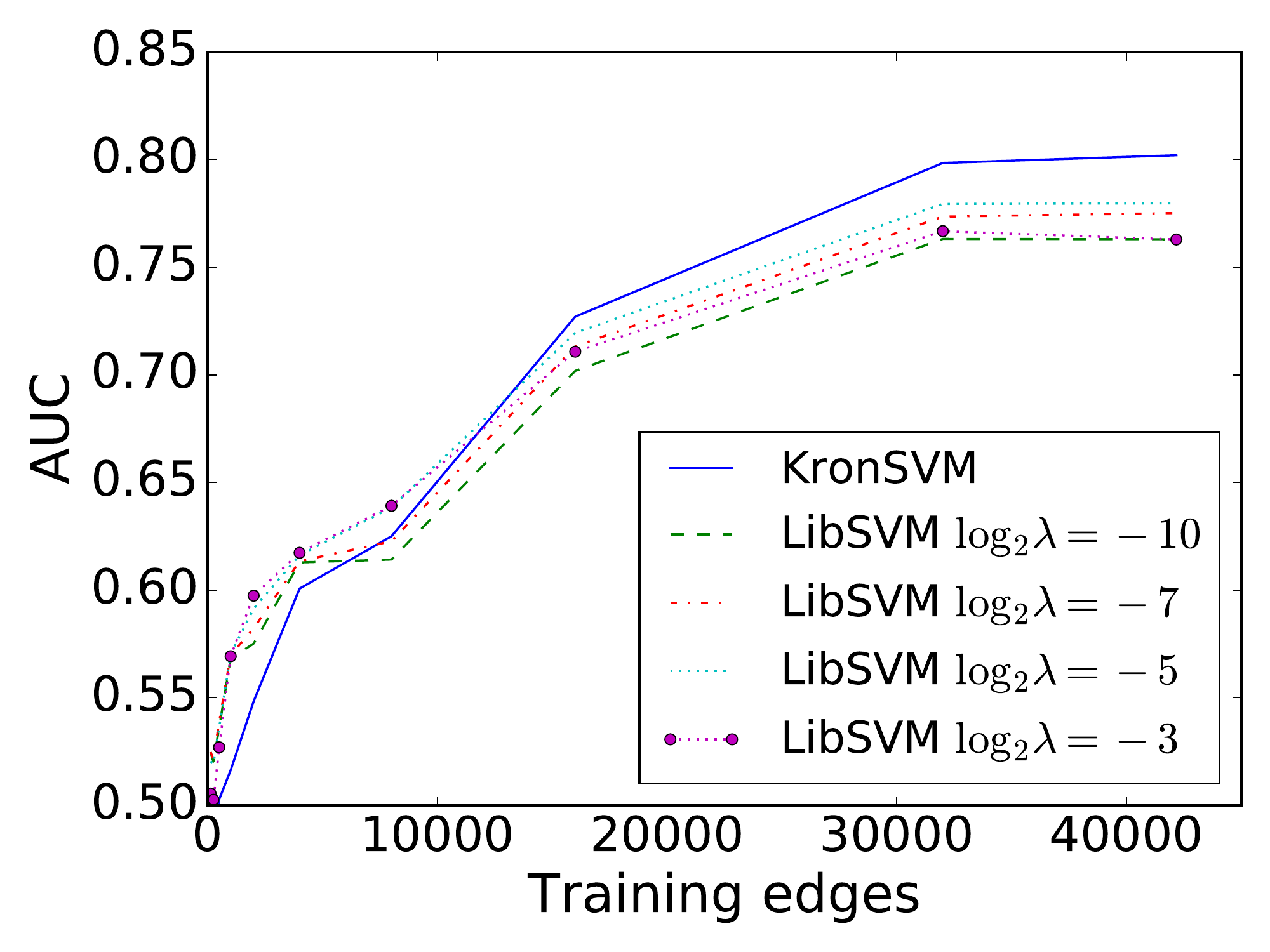}
 \caption{Drug-target experiment. Runtime comparison between KronSVM and LIBSVM (left). Prediction times for regular LibSVM decision function, and one that uses sparse Kronecker product shortcuts (middle). Cross-validated AUCs (right).}
\label{fig:libsvm_comparison}
\end{center}
\end{figure*}

\subsection{Training time}

In order to demonstrate the improvements in training speed that can be realized using the sparse Kronecker product algorithm,
we compare our Kronecker SVM algorithm to the LibSVM solver on the Ki-data set.
Based on preliminary tests, we set $\gamma=10^{-5}$, as this value produces informative (not too close to identity matrix,
or to matrix full of ones) kernel matrices for both the start and end vertices. For the Kronecker SVM, we perform 10 inner and
10 outer iterations with $\lambda=2^{-5}$. For LibSVM, we present the values for the grid $[2^{-7}, 2^{-5}, 2^{-3}, 2^{-1}]$, as
results for it vary more based on the choice of regularization parameter. We perform 9-fold cross-validation on the $K_i$ data,
for various training set sizes. In Figure~\ref{fig:libsvm_comparison} (left) we present the running times for
training the methods for different numbers of training edges, as well as the corresponding cross-validated AUCs. 
The KronSVM results, while similar as before, are not directly comparable to those in the earlier experiments due to
different kernel function being used.

As can be expected based on the computational complexity considerations, the KronSVM algorithm has superior scalability compared
to the regular SVM implemented in the LibSVM package. In the experiment on 42000 edges the difference is already 25 seconds versus
15 minutes. The LibSVM runtimes could certainly be improved for example by using earlier termination of the optimization.
Still, this would not solve the basic issue that without using computational
shortcuts such as the generalized Kronecker product algorithm (Algorithm~\ref{alg:bgkbvdetailed}) proposed in our work, one
will need to construct a significant part of the kernel matrix for the edges. For LibSVM, the scaling is roughly quadratic
in the number of edges, while for KronSVM it is linear. Considering the cross-validated
AUCs (Figure~\ref{fig:libsvm_comparison} (middle)), it can be observed that the AUC of the KronSVM with early stopping is
very much competitive with that of LibSVM.  To conclude, the proposed Kronecker product algorithm allows substantially
faster training for graph data than if one uses current state-of-the-art solvers, that are however not able to make use
of the shared structure.

\subsection{Prediction time}

Next, we show how the sparse Kronecker product algorithm can be used to accelerate predictions made for new data.
As in previous experiment, we train LibSVM on the Ki-data for varying data set sizes, with the training and test split
done as in the 9-fold cross-validation experiment. We use $\lambda=-5$, and the same kernel and kernel parameters as before.
We use the predictor learned by the SVM to make predictions for the 10000 drug-target pairs in the test set.

We compare the running times for two approaches for doing the predictions. 'Baseline' refers to the standard decision function
implemented in LibSVM. 'Kronecker', refers to implementation that computes the predictions using the sparse Kronecker
product algorithm. The 'Kronecker' method is implemented by combining the LibSVM code with additional code that
after training LibSVM, reads in the dual coefficients of the resulting predictor and generates a new predictor that
uses the shortcuts proposed in this paper. Both predictors are equivalent in the sense that they produce (within
numerical accuracy) exactly the same predictions.

The results are plotted in Figure~\ref{fig:libsvm_comparison} (right). For both methods the prediction time
increases linearly with the training set size (see equations (\ref{eq:predcomp}) and (\ref{eq:naivepredcomp})). 
However, the Kronecker method is more than 1000 times faster than Baseline,
demonstrating that significant speedups can be realized by using the sparse Kronecker product algorithm for computing
predictions.

\subsection{Large scale experiments with simulated data}

\begin{figure*}[t]
\begin{center}
\includegraphics[width=0.26\paperwidth]{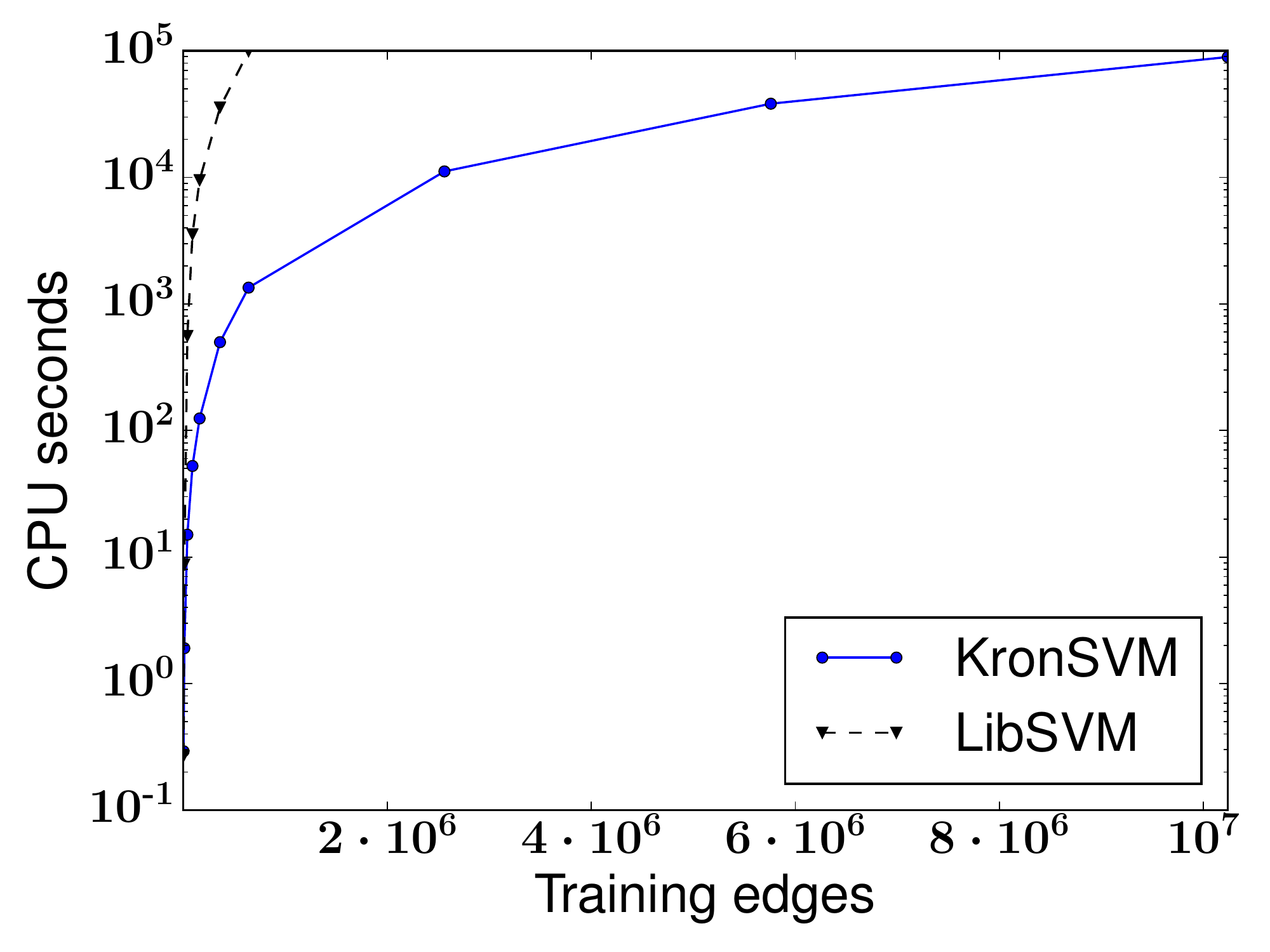}
\includegraphics[width=0.26\paperwidth]{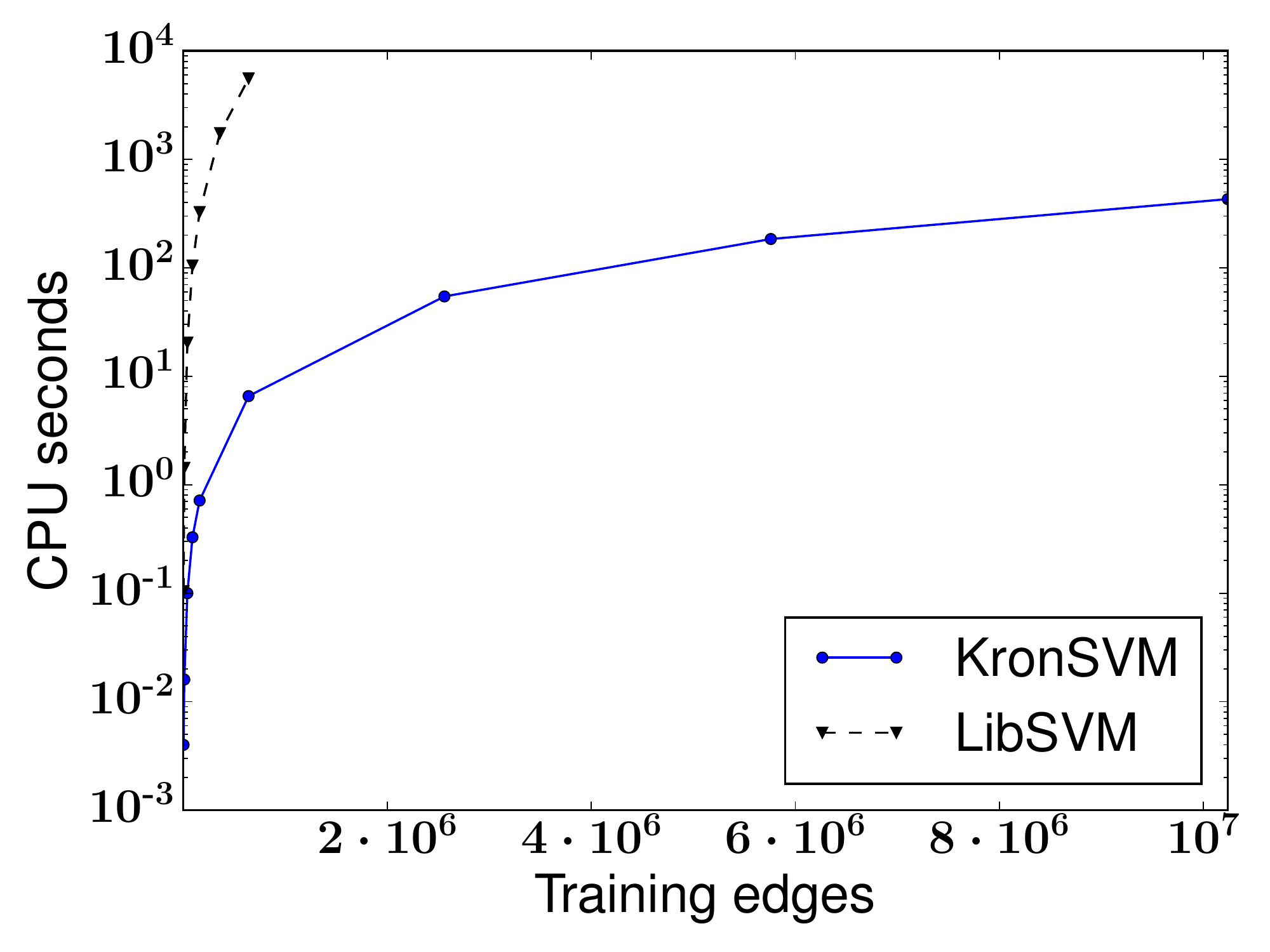}
\includegraphics[width=0.26\paperwidth]{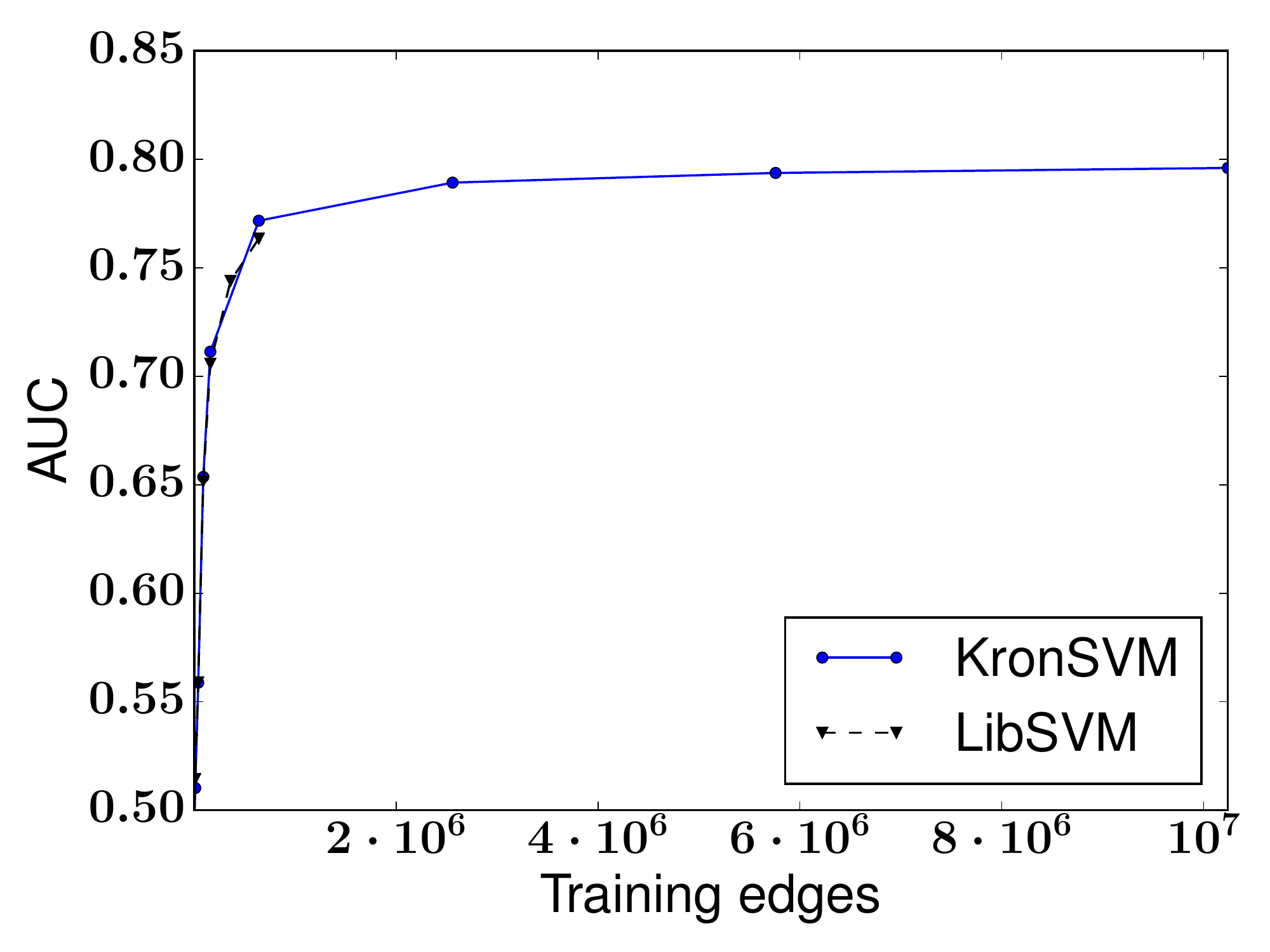}
 \caption{Checkerboard simulation. Runtime comparison between KronSVM and LIBSVM training (left) and prediction times (middle), and the corresponding test set AUCs (right).}
\label{fig:simulation_comparison}
\end{center}
\end{figure*}

In order to demonstrate the scaling of the proposed Kronecker product kernel methods to larger problem sizes than those
encountered in the drug-target data sets, we implemented scalability experiments on various sized subsets of the
checkerboard data. The training set is formed as follows. First, we generate the same number of start vertices
$\osize$ and end vertices $\qsize$, then generate  The independent test set is generated in the same way.
We use the Gaussian kernel, based on preliminary tests we set $\lambda=2^{-7}$ and $\gamma=1$, as parameters around this range allow learning
the simulated function. As before, KronSVM uses 10 inner and 10 outer iterations.
We train both KronSVM and LibSVM for varying data set sizes, up until they reach the point where training takes
more than 24 hours to complete. We also measure how long computing predictions takes for a test set of the same
size as the training set. Finally, we also measure test set AUC in order to show that the learners can really
learn to solve the simulated problem (note that due to random flipping of classes, even the optimal predictor
would have only $0.8$ AUC). The results are presented in Figure~\ref{fig:simulation_comparison}.

Again, KronSVM outperforms LibSVM by several orders of magnitude.
KronSVM can be trained in 24 hours on approximately 10 million edges (correspondingly, with 6400 start and end vertices).
The LibSVM experiments were discontinued after training on $64000$ edges (correspondingly $1600$ start and end vertices)
took more than 27 hours. For the same training set size, KronSVM can be trained in $23$ minutes. Regarding prediction times,
for LibSVM model trained with $64000$ edges it took $93$ minutes to compute predictions for a test set of same size,
whereas with the generalized Kronecker product shortcuts the same computations can be done in $7$ seconds.
KronSVM can with a model trained on 10 million edges, make predictions for a test set of also 10-million edges in $7$ minutes.
When trained on $10$ million edges, the KronSVM implementation used roughly 1.5 Gigabytes of memory.

\subsection{Comparison of graph learning methods}

Finally, we compare the proposed approach to alternative scalable graph learning approaches, such that can generalize to making predictions for edges that are vertex disjoint with the training graph. We consider the following baseline methods, that use as feature representation the concatenation $[\bm{d},\bm{t}]$ of the start and end vertex features:
\begin{itemize}
\item Linear model, stochastic gradient descent (SGD) \cite{bottou-2010}: We fit a linear model
$f(\bm{d},\bm{t})=\langle \bm{w}, [\bm{d},\bm{t}] \rangle$ to the data using
stochastic gradient descent over the edges. The approach is extremely scalable, and it is not necessary to load
all data into memory at once. Previously, \cite{menon2010loglinear} has used
sgd with logistic regression for cold start learning with recommender systems.
We consider both the logistic and the hinge loss.
\item K-nearest neighbors (KNN): KNN methods have enjoyed substantial popularity in graph
prediction applications such as biological interaction prediction \cite{laarhoven2013predicting}
and recommender systems \cite{Desrosiers2011}. The method can model highly non-linear functions and
can scale well especially to low-dimensional problems by using efficient data structures
for speeding up the neighborhood-search.
\end{itemize}
For the baseline methods, we use the implementations from the scikit-learn package \cite{scikit-learn}.
The sgd regularization parameter and KNN number of neighbors parameters are 
selected with internal 3-fold cross-validation. The number of stochastic gradient descent updates is set to $10^6$
(or at minimum one full pass through data). For KronSVM and KronRidge we set $\lambda=0.0001$, as the methods were
not very sensitive to amount of regularization when optimization was terminated early. KronSVM uses 10 inner
and 10 outer iterations on each data set, while KronRidge uses 100 iterations.
As before, for Checker data sets the Kronecker methods use Gaussian kernel with $\gamma=1$
for both vertex kernels, for other data sets we use linear vertex kernels.
For the drug-target interaction data sets we use 3x3 fold cross-validation as before, for the checkerboard data sets we generate separate test set with 6250000 edges.

\begin{table}[t]
\caption{AUCs for learning methods}
\centering
\begin{tabular}{l|llllll}
\hline
& Ki & GPCR & IC & E & Checker & Checker+\\
\hline
KronSVM & \bf{0.77} & 0.62 & 0.68 & \bf{0.79} & \bf{0.73} & \bf{0.80} \\
KronRidge & 0.75 & 0.62 & \bf{0.69} & 0.72 & 0.71 & 0.79 \\
SGD hinge & 0.76 & 0.60 & 0.63 & 0.72 & 0.50 & 0.50 \\
SGD logistic & 0.76 & \bf{0.67} & 0.64 & 0.72 & 0.50 & 0.50  \\
KNN & 0.71 & 0.63 & 0.68 & 0.70 & 0.68 & 0.79 \\
\hline
\end{tabular}
\label{tb:comparison_aucs}
\end{table}

\begin{table}[t]
\caption{CPU runtime in seconds for learning methods}
\centering
\begin{tabular}{l|llllll}
\hline
& Ki & GPCR & IC & E & Checker & Checker+\\
\hline
KronSVM & 298 & 7 & 14 & 191 & 371 & 89500 \\
KronRidge & 68 & 4 & 8 & 60 & 188 & 45150 \\
SGD hinge & 57 & 14 & 19 & 41 & 17 & 130 \\
SGD logistic & 63 & 19 & 25 & 57 & 17 & 128\\
KNN & 5554 & 43 & 267 & 26457 & 46 & 1756 \\
\hline
\end{tabular}
\label{tb:comparison_runtimes}
\end{table}

In Table~\ref{tb:comparison_aucs} we present the AUCs for the compared methods.
Overall, KronSVM performs the best, yielding the best performance
on Ki, E, Checker and Checker+ data sets. 
KronRidge performs slightly worse than KronSVM on most data sets, possibly due
to the fact that the squared loss is not as well suited for classification as the squared hinge
loss. While the linear SGD methods provide a surprisingly  competitive baseline, they do not
quite reach the performance of the best methods, and it is impossible for them to
outperform random guessing for the Checker data sets, due to the non-linearity of the task.
The KNN method performs reliably over all the data sets, but does not yield the best performance
on any of them.

Regarding the runtimes presented in Table~\ref{tb:comparison_runtimes}, KronSVM and KronRidge
strike a good balance between accurate predictions and the ability to scale to all the data sets.
The linear SGD methods provide overall the best scalability, but at the cost of not being able to model nonlinearities
in the data. The scalability of the KNN depends on the dimensionality of the data; on Checker and
Checker+ the method excels because there are only 2 features, whereas on Ki, IC, E, and GPCR, the method is not
competitive.

\section{Conclusion}

In this work, we have proposed a generalized Kronecker product algorithm. A simple optimization
framework is described in order to show how the proposed algorithm can be used to develop
efficient training algorithms for pairwise kernel methods. Both computational complexity
analysis and experiments show that the resulting algorithms can provide order of
magnitude improvements in computational efficiency both for training and making predictions,
compared to existing kernel method solvers. Further, we show that the approach compares favorably
to other types of graph learning methods. The implementations for the generalized Kronecker product 
algorithm, as well as for Kronecker product kernel learners are made freely available
under open source license.

\bibliographystyle{IEEEtran}
\bibliography{myBibliography}
\end{document}